%% file: main.tex
\documentclass{article}

     \PassOptionsToPackage{square, numbers}{natbib}
     \usepackage[preprint]{neurips_2019}

\usepackage[utf8]{inputenc} % allow utf-8 input
\usepackage[T1]{fontenc}    % use 8-bit T1 fonts
\usepackage{hyperref}       % hyperlinks
\usepackage{url}            % simple URL typesetting
\usepackage{booktabs}       % professional-quality tables
\usepackage{amsfonts}       % blackboard math symbols
\usepackage{nicefrac}       % compact symbols for 1/2, etc.
\usepackage{microtype}      % microtypography

\usepackage{xcolor}

\usepackage{amsfonts}
\usepackage[short]{optidef}
\usepackage{mathtools}
\usepackage{amsthm}
\usepackage{bbm}
\usepackage{mathrsfs}
\usepackage{verbatim}
\usepackage{enumitem}   
\usepackage[ruled,vlined]{algorithm2e}
\usepackage{wrapfig}
\usepackage{subcaption}
\usepackage{wrapfig}
\usepackage{algorithmic}

\usepackage{multicol}
\usepackage{multirow}
\usepackage{mathtools}
\usepackage{amsmath}
\DeclareMathOperator*{\argmax}{arg\,max}
\DeclareMathOperator*{\argmin}{arg\,min}

\newtheorem{innercustomthm}{Theorem}
\newenvironment{customthm}[1]
  {\renewcommand\theinnercustomthm{#1}\innercustomthm}
  {\endinnercustomthm}
\newtheorem{innercustomlem}{Lemma}
\newenvironment{customlem}[1]
  {\renewcommand\theinnercustomlem{#1}\innercustomlem}
  {\endinnercustomlem}
\newtheorem{innercustomcor}{Corollary}
\newenvironment{customcor}[1]
  {\renewcommand\theinnercustomcor{#1}\innercustomcor}
  {\endinnercustomcor}

\newcommand{\smallsubsection}[1]{\paragraph{#1}\mbox{}\\}

\newtheorem{theorem}{Theorem}[section]

\newtheorem{definition}{Definition}

\usepackage{xr}
\makeatletter
\newcommand*{\addFileDependency}[1]{
  \typeout{(#1)}
  \@addtofilelist{#1}
  \IfFileExists{#1}{}{\typeout{No file #1.}}
}
\makeatother

\newcommand*{\myexternaldocument}[1]{
    \externaldocument{#1}
    \addFileDependency{#1.tex}
    \addFileDependency{#1.aux}
}
\myexternaldocument{supplementary}

\title{Provable Certificates for Adversarial Examples: Fitting a Ball in the Union of Polytopes}

\author{%
  Matt Jordan\\ 
  University of Texas at Austin \\ 
  \texttt{mjordan@cs.utexas.edu}\\
  \And
  Justin Lewis \\ 
  University of Texas at Austin \\
  \texttt{justin94lewis@utexas.edu}\\
  \And 
  Alexandros G. Dimakis \\
  University of Texas at Austin \\ 
  \texttt{dimakis@austin.utexas.edu}\\
}

\begin{document}
\newtheorem{lemma}[theorem]{Lemma}
\newtheorem{corollary}[theorem]{Corollary}
\maketitle

\begin{abstract}

We propose a novel method for computing exact pointwise robustness of deep neural networks for all convex $\ell_p$ norms. 
Our algorithm, GeoCert, finds the largest $\ell_p$ ball centered at an input point $x_0$, within which the output class of a given neural network with ReLU nonlinearities remains unchanged. 
We relate the problem of computing pointwise robustness of these networks to that of computing the maximum norm ball with a fixed center that can be contained in a non-convex polytope. This is a challenging problem in general, however we show that there exists an efficient algorithm to compute this for polyhedral complices. Further we show that piecewise linear neural networks partition the input space into a polyhedral complex. Our algorithm has the ability to almost immediately output a nontrivial lower bound to the pointwise robustness which is iteratively improved until it ultimately becomes tight. We empirically show that our approach generates distance lower bounds that are tighter compared to prior work, under moderate time constraints.

\end{abstract}

\input{tex_files/01_introduction.tex}

\input{tex_files/02_related_work.tex}

\input{tex_files/03_centered_chebyshev.tex}

\input{tex_files/04_plnn.tex}

\input{tex_files/05_speedups.tex}

\input{tex_files/06_experiments.tex}

\input{tex_files/07_conclusion.tex} 

\bibliographystyle{plain}
\bibliography{new_references}
\newpage
\input{supplementary_arxiv.tex}
\end{document}

%% file: tex_files/01_introduction.tex
\section{Introduction}\label{sec:intro}

% Paragraph 1 : pointwise robustness 
The problem we consider in this paper is that of finding the $\ell_p$-pointwise robustness of a neural net with ReLU nonlinearities with respect to general $\ell_p$ norms. The pointwise robustness of a neural net classifier, $f$, for a given input point $x_0$ is defined as the smallest distance from $x_0$ to the decision boundary \cite{Bastani2016-zs}. Formally, this is defined as 
\begin{equation}\label{eq:pointwise-robustness}
    \rho(f, x_0, p) := \inf_{x} \{\epsilon \geq 0 \; | \; f(x) \neq f(x_0) \; \land \; ||x-x_0||_p = \epsilon\}.
\end{equation}

%Computing the pointwise robustness, or the average pointwise robustness across a test set of data points, is the central problem in certifying that neural nets are robust to adversarial attacks. Unfortunately, exactly computing the pointwise robustness for even the $\ell_\infty$ norm has been shown to be NP-complete \cite{Katz2017-qz}. Previous work that exactly computes the pointwise robustness has relied on mixed integer linear programming formulations \cite{Cheng2017-xq, Tjeng2017-qp} or SMT-based approaches \cite{Katz2017-qz, Ehlers2017-pg}, however these techniques work only for the $\ell_\infty$-norm. There have also been numerous techniques to efficiently compute a \emph{lower bound} on the pointwise robustness. These approaches employ a variety of relaxation techniques, such as duality \cite{Krishnamurthy2018-ow}, layer-wise approximations of the range of a neural net using linear or semidefinite programming\cite{Zico_Kolter2017-va, Wong2018-gw, Raghunathan2018-nu}, abstract representations with zonotopes \cite{Mirman2018-nx}, and bounding the global or local lipschitz constant of a network \cite{Hein2017-ky, Raghunathan2018-nu, Tsuzuku2018-nx, Szegedy2013-yt}. These approaches, while efficient, may provide very loose lower bounds on the robustness.

% ================================================================
%OUTLINE:  (MJ v2)
% - state problem we solve 
% - state hardness and overview of approaches 
% - state how we solve this problem 
% - tout benefits of our method 

% -- paragraph 2: state hardness and overview of approaches 
Computing the pointwise robustness is the central problem in certifying that neural nets are robust to adversarial attacks. Exactly computing this quantity this problem has been shown to be NP-complete in the $\ell_\infty$ setting \cite{Katz2017-qz}, with hardness of approximation results under the $\ell_1$ norm \cite{fastlin}. Despite these hardness results, multiple algorithms have been devised to exactly compute the pointwise robustness, though they may require exponential time in the worst case. As a result, efficient algorithms have also been developed to give provable lower bounds to the pointwise robustness, though these lower bounds may be quite loose. 

% -- paragraph 3: state how we solve this problem 
In this work, we propose an algorithm that initially outputs a nontrivial lower bound to the pointwise robustness and continually improves this lower bound until it becomes tight. Although our algorithm has performance which is theoretically poor in the worst case, we find that in practice it provides a fundamental compromise between the two extremes of complete and incomplete verifiers. This is useful in the case where a lower-bound to the pointwise robustness is desired under a moderate time budget.

% -- paragraph 4: state the general mathematical problem and results
The central mathematical problem we address is how to find the largest $\ell_p$ ball with a fixed center contained in the union of convex polytopes. We approach this by decomposing the boundary of such a union into convex components. This boundary may have complexity exponential in the dimension in the general case. However, if the polytopes form a polyhedral complex, an efficient boundary decomposition exists and we leverage this to develop an efficient algorithm to compute the largest $\ell_p$ ball with a fixed center contained in the polyhedral complex. We connect this geometric result to the problem of computing the pointwise robustness of piecewise linear neural networks by proving that the linear regions of piecewise linear neural networks indeed form a polyhedral complex. Further, we leverage the lipschitz continuity of neural networks to both initialize at a nontrivial lower bound, and guide our search to tighten this lower bound more quickly.

Our contributions are as follows:
\begin{itemize}
    \item We provide results on the boundary complexity of polyhedral complices, and use these results to motivate an algorithm to compute the the largest interior $\ell_p$ ball centered at $x_0$. 
    \item We prove that the linear regions of piecewise linear neural networks partition the input space into a polyhedral complex. 
    \item We incorporate existing incomplete verifiers to improve our algorithm and demonstrate that under a moderate time budget, our approach can provide tighter lower bounds compared to prior work.  
\end{itemize}

%% file: tex_files/02_related_work.tex
\section{Related Work}\label{sec:related}

\paragraph{Complete Verifiers:} We say that an algorithm is a \emph{complete verifier} if it exactly computes the pointwise robustness of a neural network. Although this problem is NP-Complete in general under an $\ell_\infty$ norm \cite{Katz2017-qz}, there are two main algorithms to do so. The first leverages formal logic and SMT solvers to generate a certificate of robustness \cite{Katz2017-qz}, though this approach only works for $\ell_\infty$ norms. The second formulates certification of piecewise linear neural networks as mixed integer programs and relies on fast MIP solvers to be scalable to reasonably small networks trained on MNIST \cite{Tjeng2017-qp, fischetti2018deep, lomuscio2017approach, dutta2018output, Cheng2017-xq}. This approach extends to the $\ell_2$ domain so long as the mixed integer programming solver utilized can solve linearly-constrained quadratic programs \cite{Tjeng2017-qp}. Both of these approaches are fundamentally different than our proposed method and do not provide a sequence of ever-tightening lower bounds. Certainly each can be used to certify any given lower bound, or provide a counterexample, but the standard technique to do so is unable to reuse previous computation.

\paragraph{Incomplete Verifiers:} There has been a large body of work on algorithms that output a certifiable lower bound on the pointwise robustness. We call these techniques \emph{incomplete verifiers}. These approaches employ a variety of relaxation techniques. Linear programming approaches admit efficient convex relaxations that can provide nontrivial lower bounds \cite{Zico_Kolter2017-va, fastlin, salman-barrier, Ehlers2017-pg}. Exactly computing the Lipschitz constant of neural networks has also been shown to be NP-hard \cite{virmaux2018lipschitz}, but overestimations of the Lipschitz constant have been shown to provide lower bounds to the pointwise robustness \cite{raghunathanlip,fastlin, Szegedy2013-yt, Hein2017-ky, Tsuzuku2018-nx}. Other relaxations, such as those leveraging semidefinite programming, or abstract representations with zonotopes are also able to provide provable lower bounds \cite{Raghunathan2018-nu, Mirman2018-nx}. An equivalent formulation of this problem is providing overestimations on the range of neural nets, for which interval arithmetic has been shown useful \cite{neurify, wang2018formal}. Other approaches generate lower bounds by examining only a single linear region of a PLNN \cite{feizi, croce2018provable}, though we extend these results to arbitrarily many linear regions. These approaches, while typically more efficient, may provide loose lower bounds. 

%Further, these approaches cannot compute the exact distance to the decision boundary generate a tight lower bound even with infinite computation time.

%% file: tex_files/03_centered_chebyshev.tex
\section{Centered Chebyshev Ball}\label{sec:chebyshev}
%%%%%%%%%%%%%%%%%%%%%%%%%%%%%%%%%%%%%%%%%%%%%%%%%%%%%%%%%%%%%%%
%                                                             %
%                       NOTATION PARAGRAPH                    %
%                                                             %
%%%%%%%%%%%%%%%%%%%%%%%%%%%%%%%%%%%%%%%%%%%%%%%%%%%%%%%%%%%%%%%
\smallsubsection{Notations and Assumptions} 
Before we proceed, some notation. A \textit{convex polytope} is a bounded subset of $\mathbb{R}^n$ that can be described as the intersection of a finite number of halfspaces. The polytopes we study are described succinctly by their linear inequalities (i.e., they are H-polytopes), which means that the number of halfspaces defining the polytope, denoted by $m$, is at most $\mathcal{O}(poly(n))$, i.e. polynomial in the ambient dimension. If a polytope $\mathcal{P}$ is described as $\{x \; | \; Ax \leq b\}$, an $(n-k)$-\emph{face} of $\mathcal{P}$ is a nonempty subset of $\mathcal{P}$ defined as the set $\{x \; | \; x \in \mathcal{P} \; \land \; A^=x=b^=\}$ where $A^=$ is a matrix of rank $k$ composed of a subset of the rows of $A$, and $b^=$ is the corresponding subset of $b$. We use the term \emph{facet} to refer to an $(n-1)$ face of $\mathcal{P}$.
We define the boundary $\delta \mathcal{P}$ of a polytope as the union of the facets of $\mathcal{P}$. We use the term \textit{nonconvex polytope} to describe a subset of $\mathbb{R}^n$ that can be written as a union of finitely many convex polytopes, each with nonempty interior. 
%: 
%\begin{equation}\label{eq:single-polytope-boundary}
%\delta \mathcal{P} = \bigcup_{i\in [m]} \{x \;|\; Ax \leq b \land a_i^Tx = b_i\}
%\end{equation}
%where $a_i$ is the $i^{th}$ row of $A$.
The $\ell_p$-norm ball of size $t$ centered at point $x_0$ is denoted by $B_t^p(x_0):= \{x \; | \; ||x-x_0||_p \leq t\}$. The results presented hold for $\ell_p$ norms for $p \geq 1$. When the choice of norm is arbitrary, we use $||\cdot||$ to denote the norm and $B_t(x_0)$ to refer to the corresponding norm ball.

%%%%%%%%%%%%%%%%%%%%%%%%%%%%%%%%%%%%%%%%%%%%%%%%%%%%%%%%%%%%%%%
%                                                             %
%                       SINGLE POLYTOPE SECTION               %
%                                                             %
%%%%%%%%%%%%%%%%%%%%%%%%%%%%%%%%%%%%%%%%%%%%%%%%%%%%%%%%%%%%%%%

\noindent \textbf{Centered Chebyshev Balls:}
%\smallsubsection{Centered Chebyshev Balls}
Working towards the case of a union of polytopes, we first consider the simple case of fitting the largest $\ell_p$-ball with a fixed center inside a single polytope. The uncentered version of this problem is typically referred to as finding the \emph{Chebyshev center} of a polytope and can be computed via a single linear program \cite{Boyd2004-dj, Botkin1994-ln}. When the center is fixed, this can be viewed as computing the projection to the boundary of the polytope. In fact, in the case for a single polytope, it suffices to compute the projection onto the hyperplanes containing each facet. See Appendix \ref{app:chebyshev} for further discussion computing projections onto polytopes. Ultimately, because of the polytope's geometric structure, the problem's decomposition is straightforward. This theme of efficient boundary decomposition will prove to hold true for polyhedral complices as shown in the following sections.

Now, we turn our attention to the case of finding a centered Chebyshev ball inside a general nonconvex polytope. This amounts to computing the projection to the boundary of the region. The key idea here is that the boundary of a nonconvex polytope can be described as the union of finitely many $(n-1)$-dimensional polytopes; however, the decomposition may be quite complex. We define this set formally as follows:
\begin{definition}\label{def:boundary}
The \textbf{boundary} of a non-convex polytope $P$ is the largest set $T\subseteq P$ such that every point $x\in T$ satisfies the following two properties:
\begin{enumerate}[label=(\roman*)]
    \item There exists an $\epsilon_0$ and a direction $u$ such that for all $\epsilon \in (0, \epsilon_0)$, there exists a neighborhood centered around $x + \epsilon u$ that is contained in $P$. \label{item:inbound}
    \item There exists an $\eta_0$ and a direction $v$ such that for all $\eta \in (0, \eta_0)$, $x+\eta v \notin P$. \label{item:outbound}
\end{enumerate}
\end{definition}

The boundary is composed of finitely many convex polytopes, and computing the projection to a single convex polytope is an efficiently computable convex program. If there exists an efficient decomposition of the boundary of a nonconvex polytope into convex sets, then a viable algorithm is to simply compute the minimal distance from $x_0$ to each component of the boundary and return the minimum. Unfortunately, for general nonconvex polytopes, there may not be an efficient convex decomposition. See Theorem \ref{thm:hyp_arrang} in Appendix \ref{app:boundary}.

However, there do exist classes of nonconvex polytopes that admit a convex decomposition with size that is no larger than the description of the nonconvex polytope itself. To this end, we introduce the following definition (see also Ch. 5 of \cite{Ziegler1995-fl}):
\begin{definition}
A nonconvex polytope, described as the union of elements of the set $\mathscr{P}=\{\mathcal{P}_1, ..., \mathcal{P}_k\}$ forms a \textbf{polyhedral complex} if, for every $\mathcal{P}_i, \mathcal{P}_j\in \mathscr{P}$ with nonempty intersection, $\mathcal{P}_i\cap\mathcal{P}_j$ is a face of both $\mathcal{P}_i$ and $\mathcal{P}_j$. Additionally, for brevity, if a pair of polytopes $\mathcal{P}, \mathcal{Q}$, form a polyhedral complex, we say they are $\textbf{PC}$. (See Figure \ref{fig:polytope_examples} for examples.)
\end{definition}
\begin{figure}
    \centering
    \includegraphics[scale=0.35]{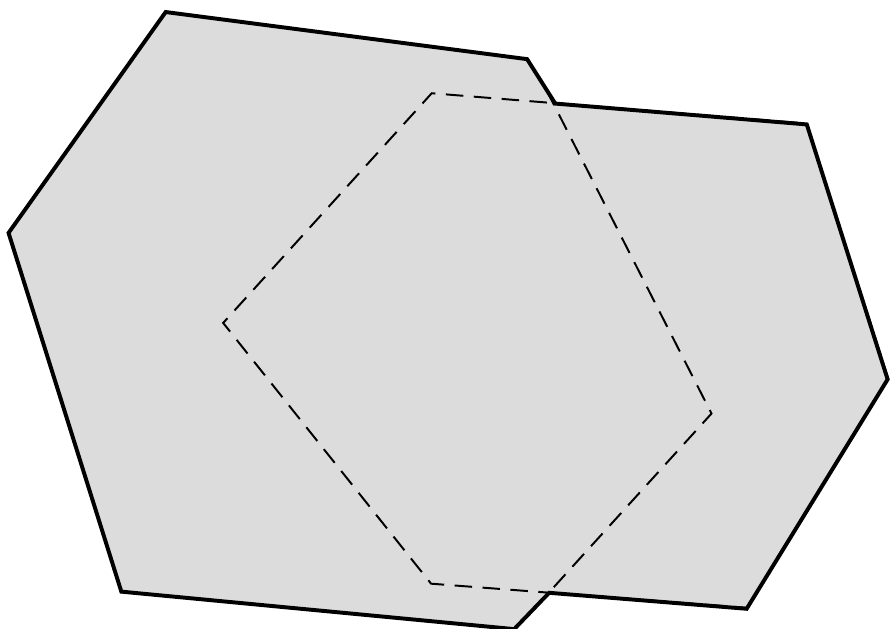}
    \includegraphics[scale=0.35]{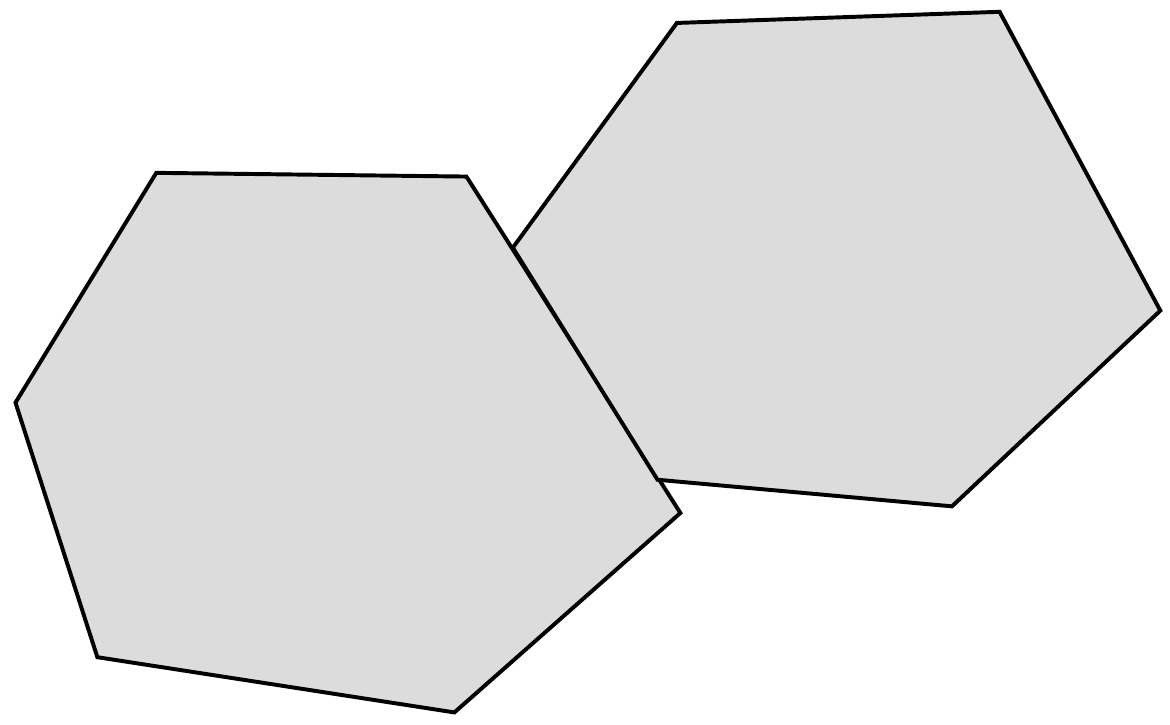}
    \includegraphics[scale=0.35]{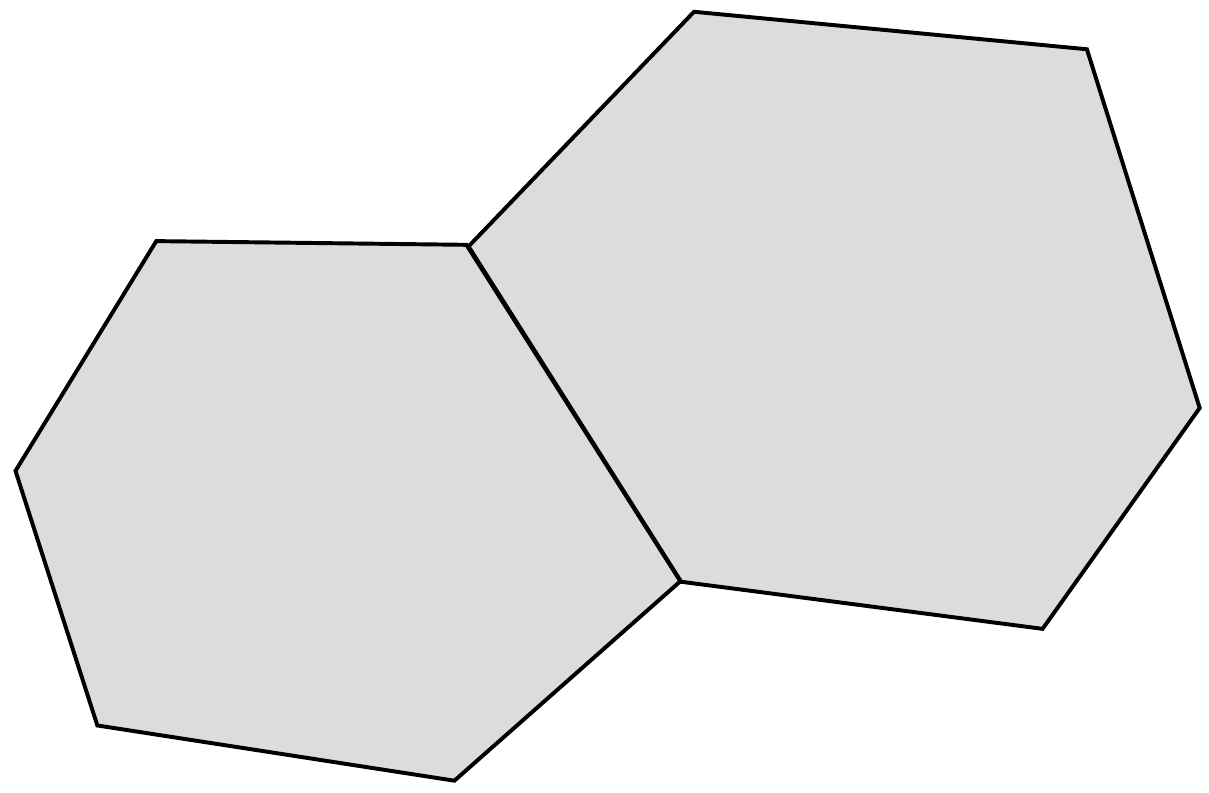}

    \caption{Three potential configurations of a nonconvex polytope. Note that only the rightmost nonconvex polytope forms a polyhedral complex.}
    \label{fig:polytope_examples}
\end{figure}

We can now state our main theorem concerning the computation of the centered Chebyshev ball within polyhedral complices:
\begin{theorem}\label{thm:main_algo}
Given a polyhedral complex, $\mathscr{P} =\{\mathcal{P}_1, \dots \mathcal{P}_k\}$, where $\mathcal{P}_i$ is defined as the intersection of $m_i$ closed halfspaces. Let $M=\sum_i m_i$, and let $x_0$ be a point contained by at least one such $\mathcal{P}_i$. Then the boundary of $\bigcup_{i \in [k]} \mathcal{P}_i$ is represented by at most $M$ $(n-1)$-dimensional polytopes. There exists an algorithm that can compute this boundary in $\mathcal{O}(poly(n, M, k))$ time. 
\end{theorem}

Returning to our desired application, we now prove a corollary about the centered Chebyshev ball contained in a union of polytopes.
\begin{corollary}
Given a collection, $\mathscr{P}=\{\mathcal{P}_1, \dots \mathcal{P}_k\}$ that meets all the conditions outlined in theorem \ref{thm:main_algo}, with the boundary of $\mathscr{P}$ computed as in theorem \ref{thm:main_algo}, the centered Chebyshev ball around $x_0$ has size
\begin{equation}\label{eq:algo-answer}
    t:= \inf_{x \in T} ||x-x_0|| 
\end{equation}
This can be solved by at most $M$ linear programs in the case of $\ell_\infty$ norm, or at most $M$ linearly constrained quadratic programs in the case of the $\ell_2$-norm. 
\end{corollary}

% \begin{figure}[h!]
%     \centering
%     \subfigure{\includesvg[width=0.6\columnwidth,height=0.6\textheight]
%     {images/svgs/union_polytopes_algo.svg}}
%     \subfigure{\includesvg[width=0.6\columnwidth,height=0.6\textheight]
%     {images/svgs/union_polytopes_algo_2.svg}}
%     \subfigure{\includesvg[width=0.6\columnwidth,height=0.6\textheight]
%     {images/svgs/union_polytopes_algo_3.svg}}
%     \caption{}
%     \label{fig:algorithm_example}
% \end{figure}

%%%%%%%%%%%%%%%%%%%%%%%%%%%%%%%%%%%%%%%%%%%%%%%%%%%%%%%%%%%%%%%
%                                                             %
%                       GRAPH THEORY + ALGORITHM              %
%                                                             %
%%%%%%%%%%%%%%%%%%%%%%%%%%%%%%%%%%%%%%%%%%%%%%%%%%%%%%%%%%%%%%%
%\smallsubsection{Graph Theoretic Formulation:} 
\noindent \textbf{Graph Theoretic Formulation:}
Theorem \ref{thm:main_algo} and its corollary provide a natural algorithm to computing the centered Chebyshev ball of a polyhedral complex: compute the convex components of the boundary and then compute the projection to each component. However one can hope to do better; one may not have to compute the distance to \emph{every} boundary facet. In the absence of other information, one must at least compute the projection to every facet intersecting the centered Chebyshev ball. 

A more natural way to view this problem is as a search problem along a bipartite graph, composed of a set of left vertices $\mathscr{F}$, right vertices $\mathscr{P}$ and edges $E$ connecting them. The left vertices represent all facets inside the polyhedral complex, and the right vertices represent all $n$-dimensional polytopes of the polyhedral complex. An edge between a face $\mathcal{F}$ and a polytope $\mathcal{P}$ exists iff $\mathcal{F}$ is a facet of $\mathcal{P}$. In other words, the graph of interest is composed of the terminal elements of the face lattice and their direct ancestors. For any polyhedral complex, the left-degree of this graph is at most 2.

The boundary facets, $T$, are some subset of $\mathscr{F}$ and our algorithm aims only to return the minimal distance between $x_0$ and $T$. Since $x_0$ exists in at least one polytope $\mathcal{P}_0 \in \mathscr{P}$, we can search locally outward starting at $\mathcal{P}_0$. For reasons that will become clear later, we rephrase distances as `potentials' and aim to find the element of $T$ which has the lowest potential. We define the pointwise potential, some function mapping $\mathbb{R}^n\rightarrow \mathbb{R}$ as $\phi(y)$, and the facet-wise potential, mapping facets to $\mathbb{R}$ as $\Phi(\mathcal{F}):= \min_{y\in\mathcal{F}} \phi(y)$. For now, one can view $\phi$ as the distance function under an $\ell_p$ norm, $\phi(y):=||x_0 - y||$. 

Algorithm 1 presents a technique to compute the centered Chebyshev ball as a search algorithm. The idea is to maintain a set of `frontier facets' in a priority queue, ordered by their potential $\Phi$. At each iteration we pop the frontier facet with minimal potential. Since this facet has left degree 2, at most one of its neighboring polytopes must not have yet been explored. If such a polytope exists, for each of its facets we add the facet and its potential $\Phi$ to the priority queue. A formal proof of correctness for using the distance function as a potential can be found in Corollary \ref{cor:lp-potential} in Appendix \ref{app:potential_proofs}.

\begin{figure}
\centering
    \begin{subfigure}{0.5\textwidth}
    \begin{algorithm}[H] 
        \SetAlgoLined
        \caption{Algorithm 1: GeoCert}
            \textbf{Input:} point $x_0$, potential $\Phi$\;
            \textbf{Initialization:} \;
            // Setup priority queue, seen-polytope set\;
            $Q \gets [\;]; C \gets \{\mathcal{P}(x_0)\}$\; 
            // Handle first polytope's facets\;
            \For{Facet $\mathcal{F}\in N(\mathcal{P}(x_0))$}{
                $Q.push((\Phi(\mathcal{F}), \mathcal{F}))$\;
            }
            // Loop until boundary is popped\;
            \While{$Q \neq \emptyset$}{
            $\mathcal{F} \gets Q.pop()$\; 
            \eIf{$\mathcal{F}$ is boundary}{
            \textbf{Return} $\mathcal{F}$\;
            }{
            \For{$\mathcal{P} \in N(\mathcal{F}) \setminus C$}{
            \For{$\mathcal{F}\in N(P)$}{
                $Q.push((\Phi(\mathcal{F}), \mathcal{F})$\;
            }
            }
            }}
    %\label{alg:geocert}
    \end{algorithm} 
    \end{subfigure}% need this comment symbol to avoid overfull hbox
    \begin{subfigure}{.5\textwidth}
    \centering
        \includegraphics[scale=0.25]{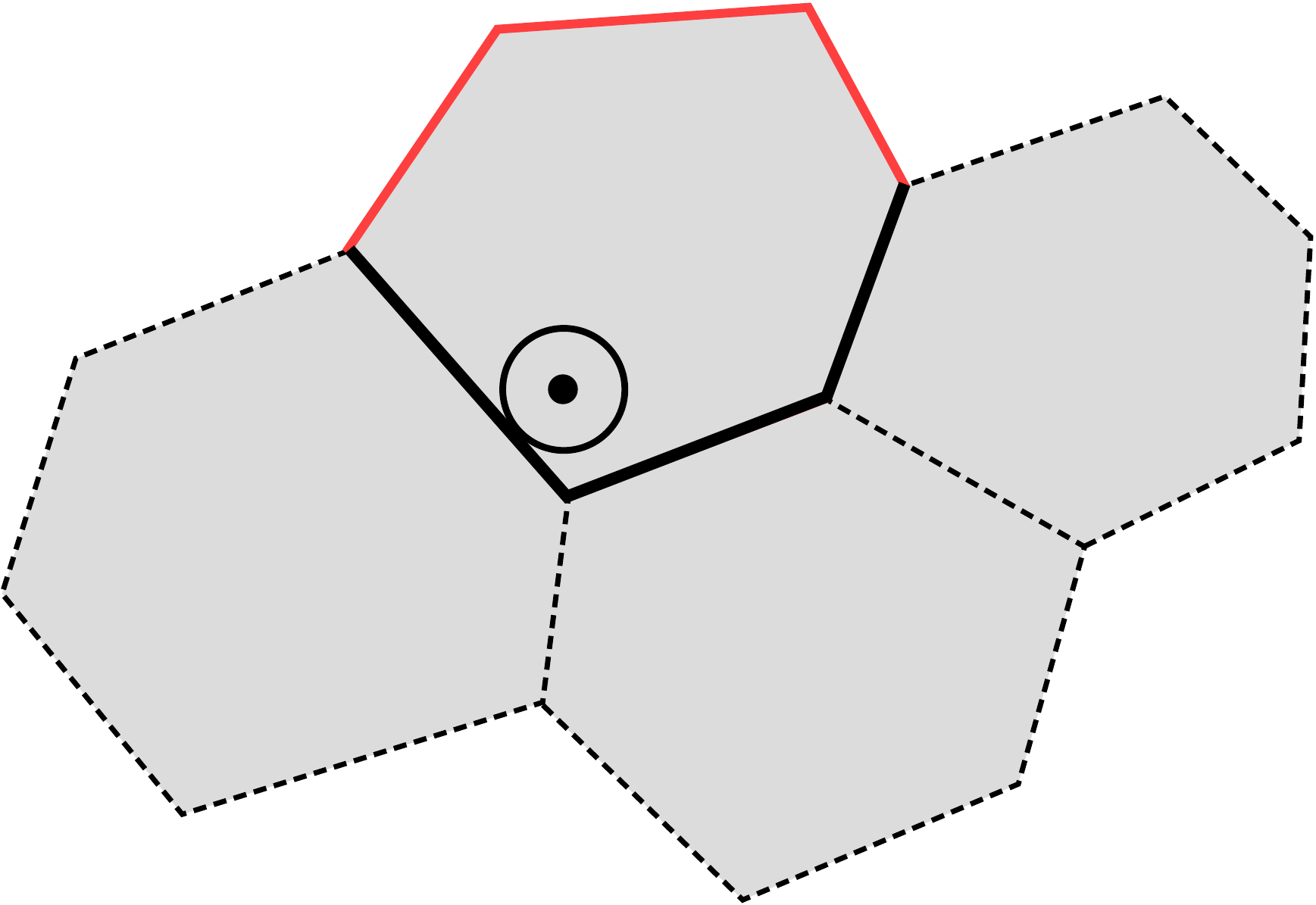}\\
        \includegraphics[scale=0.25]{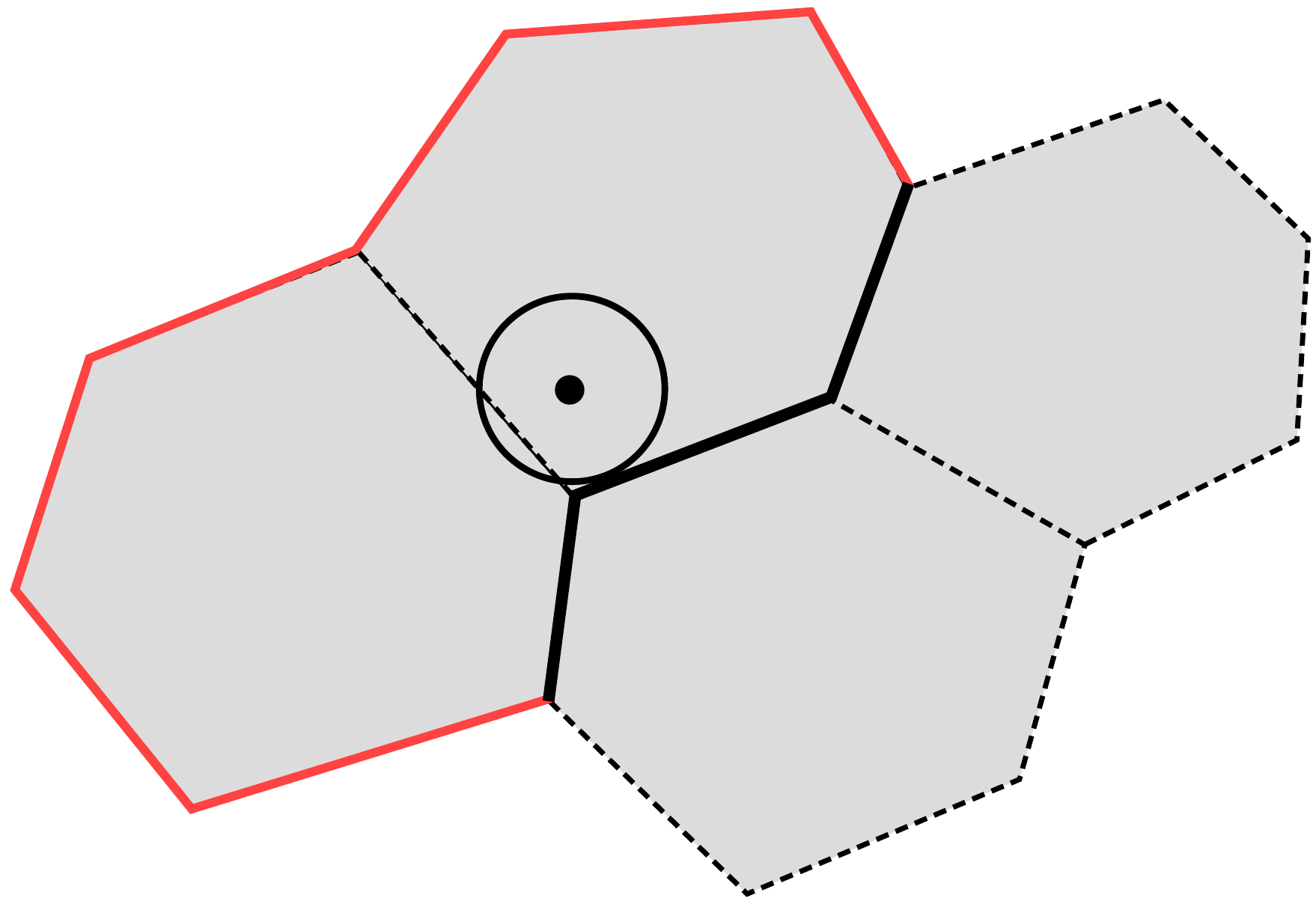}\\
        \includegraphics[scale=0.25]{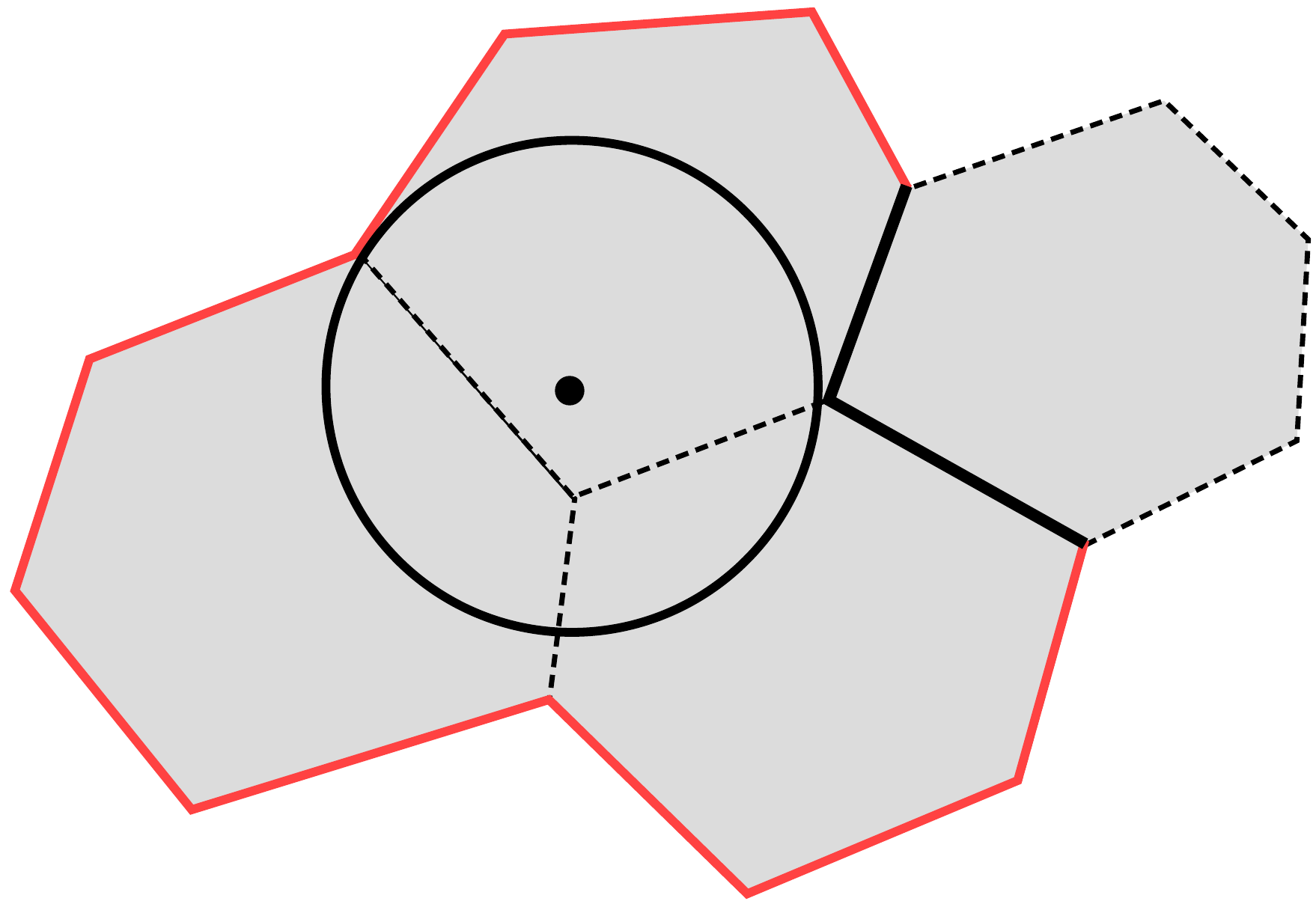}\\
    \end{subfigure}
\caption{Pseudocode for GeoCert (left) and a pictorial representation of the algorithm's behavior on a simple example (right). The facets colored belong to the priority queue, with red and black denoting adversarial facets and non-adversarial facets respectively. Once the minimal facet in the queue is adversarial, the algorithm stops.}
\end{figure}

%%%%%%%%%%%%%%%%%%%%%%%%%%%%%%%%%%%%%%%%%%%%%%%%%
%                                               %
%               HYPERPLANE LEMMAS               %
%                                               %
%%%%%%%%%%%%%%%%%%%%%%%%%%%%%%%%%%%%%%%%%%%%%%%%%
\smallsubsection{Iteratively Constructing Polyhedral Complices} 
Finally, we note an approach by which polyhedral complices may be formed that will become useful when we discuss PLNN's in the following section. We present the following three lemmas which relate to iterative constructions of polyhedral complices. Informally, they state that given any polytope or pair of polytopes which are PC, a slice with a hyperplane or a global intersection with a polytope generates a set that is still PC. 

\begin{lemma}\label{lemma:hyperplane-glue}
Given an arbitrary polytope $\mathcal{P}:= \{x \; | \; Ax \leq b\}$ and a hyperplane $\mathcal{H}:= \{x \; | \; c^Tx = d\}$ that intersects the interior of $\mathcal{P}$, the two polytopes formed by the intersection of $\mathcal{P}$ and the each of closed halfpsaces defined by $\mathcal{H}$ are PC.
\end{lemma}

\begin{lemma}
\label{lemma:double-hyperplane-glued}
Let $\mathcal{P},\mathcal{Q}$ be two PC polytopes and let $H_\mathcal{P}$, $H_\mathcal{Q}$ be two hyperplanes that define two closed halfspaces each, $H^+_\mathcal{P}, H^-_\mathcal{P}, H^+_\mathcal{Q}, H^-_\mathcal{Q}$. If $\mathcal{P}\cap\mathcal{Q}\cap H_\mathcal{P} =\mathcal{P}\cap\mathcal{Q}\cap H_\mathcal{Q}$ then the subset of the four resulting polytopes $\{\mathcal{P}\cap H^+_\mathcal{P}, \mathcal{P}\cap H^-_\mathcal{P}, \mathcal{Q}\cap H^+_\mathcal{Q}, \mathcal{Q}\cap H^-_\mathcal{Q}\}$ with nonempty interior forms a polyhedral complex.
\end{lemma}

And the following will be necessary when we handle the case where we wish to compute the pointwise robustness for the image classification domain, where valid images are typically defined as vectors contained in the hypercube $[0, 1]^n$. 

\begin{lemma}\label{lemma:hyperbox-pc}
Let $\mathscr{P} = \{\mathcal{P}_1, \dots \mathcal{P}_k\}$ be a polyhedral complex and let $\mathcal{D}$ be any polytope. Then the set $\{\mathcal{P}_i \cap \mathcal{D} \; | \; \mathcal{P}_i \in \mathscr{P}\}$ also forms a polyhedral complex.
\end{lemma}

\begin{wrapfigure}{R}{0.40\textwidth}
    \centering
    \vspace{-20pt}
    \includegraphics[width=\linewidth]{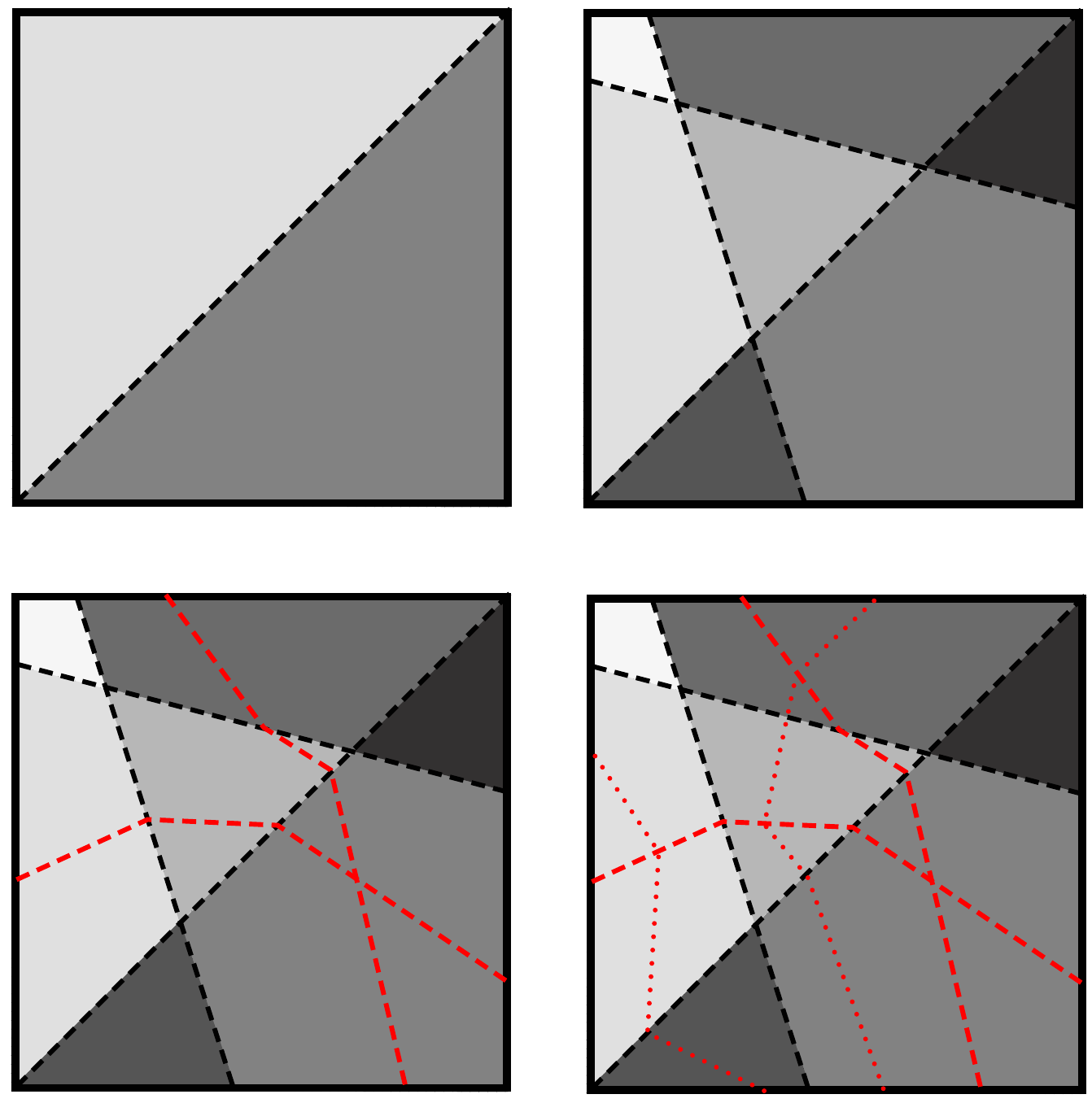}
    \vspace{-10pt}
    \caption{ Pictorial reference for proof of Theorem \ref{thm:relu_nets_are_PG}. \emph{(Top Left)} A single Relu activation partitions the input space into two PC polytopes \emph{(Top Right)} as additional activations are added at the first layer, the collection is still PC by Lemma \ref{lemma:double-hyperplane-glued}. \emph{(Bottom Left)} as the next layer of activations are added, the partitioning is linear within each region created previously and PC at the previous boundaries, thus still PC. \emph{(Bottom Right)} the partitioning due to all subsequent layers preserves PC-ness by induction.}
    \vspace{-25pt}
\end{wrapfigure}

% This can be pushed to the appendix -mj
% Finally we note that for the problem of a centered Chebyshev ball inside a union of only two polytopes, if their intersection is sufficiently low dimension, one polytope is irrelevant.
% \begin{lemma}\label{lemma:intersection-check}
% Let $\mathcal{P}$, $\mathcal{Q}$ be polytopes whose intersection is $(n-d)$ dimensional, for some $d\geq 2$, and let $x_0 \in \mathcal{P}$, with $B_t(x_0)$ the largest $\ell_p$-norm ball centered at $x_0$ contained in $\mathcal{P}\cup \mathcal{Q}$. Then $B_t(x_0)$ is contained entirely in $\mathcal{P}$.
% \end{lemma}

%% file: tex_files/04_plnn.tex
\section{Piecewise Linear Neural Networks}\label{sec:plnn}

We now demonstrate an application of the geometric results described above to certifying robustness of neural nets. We only discuss networks with fully connected layers and ReLU nonlinearities, but our results hold for networks with convolutional and skip layers as well as max and average pooling layers. 
Let $f$ be an arbitrary $L$-layer feed forward neural net with fully connected layers and ReLU nonlinearities, where each layer $f^{(i)}: \mathbb{R}^{n_{i-1}} \rightarrow \mathbb{R}^{n_{i}}$ has the form 

  \begin{equation}
  f^{(i)}(x) =
\begin{cases}
W_i x + b_i, & \text {if i = 1}\\
W_i \sigma(f^{(i-1)}(x)) +b_i, & \text{if } i > 1  \\
\end{cases}
  \end{equation}

where $\sigma$ refers to the element-wise ReLU operator. And we denote the final layer output $f^{(L)}(x)$ as $f(x)$. We typically use the capital $F(x)$ to refer to the maximum index of $f$: $F(x):= \argmax_i f_i(x)$. We define the $\emph{decision region}$ of $f$ at $x_0$ as the set of points for which the classifier returns the same label as it does for $x_0$: $\{x\; | \; F(x) = F(x_0)\}$. 

It is important to note is that $f^{(i)}(x)$ refers to the \textit{pre-ReLU} activations of the $i^{th}$ layer of $f$. Let $m$ be the number of neurons of $f$, that is $m=\sum_{i=1}^{L-1} n_i$. We describe a neuron configuration as a ternary vector, $A\in \{-1, 0, 1\}^m$, such that each coordinate of $A$ corresponds to a particular neuron in $f$. In particular, for neuron $j$, 
\begin{equation}
A_j =
\begin{cases}
+1, & \text{if neuron $j$ is `on'}\\
-1, & \text{if neuron $j$ is `off'}\\
0, & \text{if neuron $j$ is both `on' and `off'}\\
\end{cases}    
\end{equation}

Where a neuron being `on' corresponds to the pre-ReLU activation is at least zero, `off' corresponds to the pre-ReLU being at most zero, and if a neuron is both on and off its pre-ReLU activation is identically zero. Further each neuron configuration corresponds to a set 
$$\mathcal{P}_A = \{x \; | \; f(x) \text{ has neuron activation consistent with } A\}$$

% Also can just be stated and moved to the appendix 

The following have been proved before, but we include them to introduce notational familiarity:
\begin{lemma}\label{lemma:neural}
For a given neuron configuration $A$, the following are true about $\mathcal{P}_A$,
\begin{enumerate} [label=(\roman*)]
    \item  $f^{(i)}(x)$ is linear in $x$ for all $x\in \mathcal{P}_A$\label{neural:i}. 
    \item $\mathcal{P}_A$ is a polytope.
    \label{neural:ii}
\end{enumerate}
\end{lemma}

This lets us connect the polyhedral complex results from the previous section towards computing the pointwise robustness of PLNNs. Letting the potential $\phi$ be the $\ell_p$ distance, we can apply Algorithm 1 towards this problem.

\begin{theorem}
The collection of $\mathcal{P}_A$ for all $A$, such that $\mathcal{P}_A$ has nonempty interior forms a polyhedral complex. Further, the decision region of $F$ at $x_0$ also forms a polyhedral complex. 
\label{thm:relu_nets_are_PG}
\end{theorem}

In fact, except for a set of measure zero over the parameter space, the facets of each such linear region correspond to exactly one ReLU flipping configurations:
\begin{corollary}\label{cor:relu-hamming}
If the network parameters are in general position and $A,B$ are neuron configurations such that $dim(\mathcal{P}_A)=dim(\mathcal{P}_B)=n$ and their intersection is of dimension $(n-1)$, then $A,B$ have hamming distance 1 and their intersection corresponds to exactly one ReLU flipping signs.
\end{corollary}

%% file: tex_files/05_speedups.tex
\section{Speedups}

\begin{wrapfigure}{L}{0.40\textwidth}
    \centering
    \vspace{-20pt}
    \includegraphics[width=\linewidth]{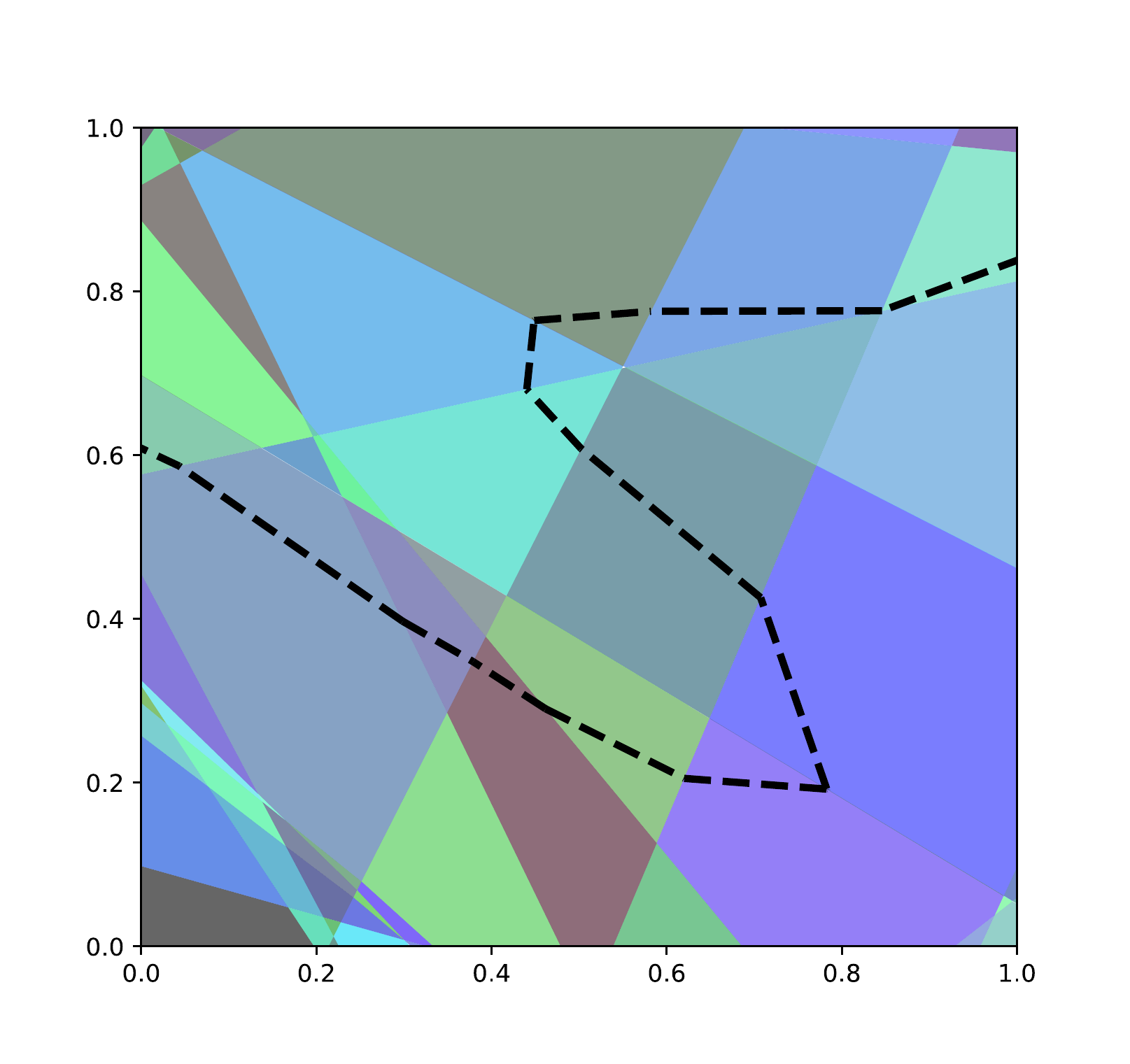}
    \vspace{-10pt}
    \caption{Piecewise Linear Regions of a 2D toy network. The dotted line represents the decision boundary.}
    \vspace{-15pt}
\end{wrapfigure}

While our results in section \ref{sec:chebyshev} hold for general polyhedral complices, we can boost the performance of GeoCert by leveraging additional structure of PLNNs. As the runtime of GeoCert hinges upon the total number of iterations and time per iteration, we discuss techniques to improve each.
\smallsubsection{Improving Iteration Speed Via Upper Bounds}\label{sec:speedups}
At each iteration, GeoCert pops the minimal element from the priority queue of `frontier facets' and, using the graph theoretic lens, considers the facets in its two-hop neighborhood. Geometrically this corresponds to popping the minimal-distance facet seen so far, considering the polytope on the opposite side of that facet and computing the distan1ce to each of its facets. In the worst case, the number of facets of each linear region is the number of ReLU's in the PLNN. While computing the projection requires a linear or quadratic program, as we will show, it is usually not necessary to compute a convex program for each every nonlinearity at every iteration.

If we can quickly guarantee that a potential facet is infeasible within the domain of interest then we avoid computing the projection exactly. In the image classification domain, the domain of valid images is usually the unit hypercube. If an upper bound on the pointwise robustness, $U$, is known, then it suffices to restrict our domain to $\mathcal{D}' := B_U(x_0)\cap \mathcal{D}$. This aids us in two ways: (i) if the hyperplane containing a facet does not intersect $\mathcal{D}'$ then the facet also does not intersect $\mathcal{D}'$; (ii) a tighter restriction on the domain allows for tighter bounds on pre-ReLU activations. For point (i), we observe that computing the feasibility of the intersection of a hyperplane and hyperbox is linear in the dimension and hence many facets can very quickly be deemed infeasible. For point (ii), if we can guarantee ReLU stability, by Corollary \ref{cor:relu-hamming}, then we can deem the facets corresponding to the each stable ReLU as infeasible. ReLU stability additionally provides tighter upper bounds on the Lipschitz constants of the network. 

Any valid adversarial example provides an upper bound on the pointwise robustness. Any point on any facet on the boundary of the decision region also provides an upper bound. In Appendix \ref{app:upper_bounds}, we describe a novel tweak that can be used to generate adversarial examples tailored to be close to the original point. Also, during the runtime of Geocert, any time a boundary facet is added to the priority queue, we update the upper bound based on the projection magnitude to this facet.

\smallsubsection{Improving Number of Iterations Via Lipschitz Overestimation}
When one uses distance as a potential function, if the true pointwise robustness is $\rho$, then GeoCert must examine every polytope that intersects $B_\rho(x_0)$. This is necessary in the case when no extra information is known about the polyhedral complex of interest. However one can incorporate the lipschitz-continuity of a PLNN into the potential function $\phi$ to reduce on the number of linear regions examined. The main idea is that as the network has some smoothness properties, any facet for which the classifier is very confident in its answer must be very far from the decision boundary. 
\begin{theorem}\label{thm:lipschitz-potential}
Letting $F(x_0) = i$, and $g_j(x) = f_i(x) - f_j(x)$ and an upper bound $L_j$ on the lipschitz continuity of $g_j$, using $\phi_{lip}(y):=||x_0 - y|| + \min_{j\neq i} \frac{g_j(y)}{L_j}$ as a potential for GeoCert maintains its correctness in computing the pointwise robustness.
\end{theorem}
The intuition behind this choice of potential is that it biases the set of seen polytopes to not expand too much in directions for which the distance to the decision boundary is guaranteed to be large. This effectively is able to reduce the number of polytopes examined, and hence the number of iterations of geocert, while still maintaining complete verification. A critical bonus of this approach is that it allows one to 'warm-start' GeoCert with a nontrivial lower bound that will only increase until becoming tight at termination. A more thorough discussion on upper-bounding the lipschitz constant of each $g_j$ can be found in \cite{fastlin}.

%% file: tex_files/06_experiments.tex
\section{Experiments}\label{sec:experiments}

%\subsection{Exactly Computing the Pointwise Robustness}
\noindent \textbf{Exactly Computing the Pointwise Robustness:}
Our first experiment compares the average pointwise robustness bounds provided by two complete verification methods, GeoCert and MIP, as well as an incomplete verifier, Fast-Lip. The average $\ell_p$ distance returned by each method and the average required time (in seconds) to achieve this bound are provided in Table \ref{table:full}. Verification for $\ell_2$ and $\ell_\infty$ robustness was conducted for 128 random validation images for two networks trained on MNIST. Networks are divided into binary and non-binary examples. Binary networks were trained to distinguish a subset of 1's and 7's from the full MNIST dataset. All networks were trained with $\ell_1$ weight regularization with $\lambda$ set to $2\times 10^{-3}$. All networks are composed of fully connected layers with ReLU activations. The layer-sizes for the two networks are as follows: i) [784, 10, 50, 10, 2] termed 70NetBin and ii) [784, 20, 20, 2] termed 40NetBin.

From Table \ref{table:full}, it is clear that Geocert and MIP return the exact robustness value while Fast-Lip provides a lower bound. While the runtimes for MIP are faster than those for GeoCert, they are within an order of magnitude. In these experiments, we record the timing when each method is left to run to completion; however, in the experiment to follow we demonstrate that GeoCert provides a non-trivial lower bound faster than other methods. 

\begin{table}[ht]

\caption{(Left) Times (seconds) to compute exact pointwise robustness on binary MNIST networks for both the $\ell_2$ and $\ell_\infty$ settings over 128 random examples. Boldface corresponds to the exact pointwise robustness. (Right) Provable lower bounds for a binary MNIST network under a fixed 300s time limit. Note that GeoCert initializes at the bound provided by Fast-Lip and continually improves. Boldface here corresponds to the tightest lower bound found. Note that our algorithm outperforms all previous methods for this task. }

    \begin{subtable}{.5\linewidth}
      \centering

\begin{tabular}{@{}ll|rr|rr@{}}
\toprule
                              &                                & \multicolumn{2}{l|}{70NetBin}     & \multicolumn{2}{l|}{40NetBin} \\ \midrule
\multicolumn{1}{c|}{Method}   & $\ell_p$                       & \multicolumn{1}{l}{Dist.} & Time  & Dist.         & Time          \\ \midrule
\multicolumn{1}{l|}{Fast-Lip} & \multirow{3}{*}{$\ell_\infty$} & 0.136                     & 0.010 & 0.132         & 0.007         \\
\multicolumn{1}{l|}{GeoCert}  &                                & \textbf{0.191}                     & 1.300 & \textbf{0.187}         & 4.031         \\
\multicolumn{1}{l|}{MIP}      &                                & \textbf{0.191}                     & 0.947 & \textbf{0.187}         & 0.689         \\ \midrule
\multicolumn{1}{l|}{Fast-Lip} & \multirow{3}{*}{$\ell_2$}      & 1.289                     & 0.010 & 1.319         & 0.013         \\
\multicolumn{1}{l|}{GeoCert}  &                                & \textbf{1.555}                     & 4.789 & \textbf{1.607}         & 21.852        \\
\multicolumn{1}{l|}{MIP}      &                                & \textbf{1.555}                     & 4.030 & \textbf{1.607}         & 5.831         \\ \bottomrule

\end{tabular}
    \end{subtable}%
    \begin{subtable}{.5\linewidth}
      \centering
\begin{tabular}{@{}l|rrr@{}}       %\label{table:exp_two}

\toprule
        & \multicolumn{3}{c}{Method} \\ \midrule
Ex. & Fast-Lip  & GeoCert & MIP \\\midrule 
1       & 1.782     & \textbf{2.251}    & 2.0 \\
2       & 1.319     & \textbf{1.356}    & 1.0 \\
3       & 1.501     & \textbf{1.620}    & 1.0 \\
4       & 1.975     & \textbf{2.499}    & 2.0 \\
5       & 1.871     & \textbf{2.402}    & 2.0 \\ \bottomrule

\end{tabular}
    \end{subtable} 
        \label{table:full}

\end{table}

\noindent \textbf{Best Lower Bound Under a Time Limit:}
%
%\subsection{Best Lower Bound Under a Time Limit}
To demonstrate the ability of GeoCert to provide a lower bound greater than those generated by incomplete verifiers and other complete verifiers under a fixed time limit we run the following experiment. On the binary MNIST dataset, we train a network with layer sizes [784, 20, 20, 20, 2] using Adam and a weight decay of 0.02 \cite{kingma2014adam}. We allow a time limit of 5 minutes per example, which is not sufficient for either GeoCert or MIP to complete. As per the codebase associated with \cite{Tjeng2017-qp}, for MIP we use a binary search procedure of $\epsilon =[0.5, 1.0, 2.0, 4.0, \dots]$ to verify increasingly larger lower bounds. We also compare against the lower bounds generated by Fast-Lip \cite{fastlin}, noting that using the Lipschitz potential described in Section \ref{sec:speedups} allows GeoCert to immediately initialize to the bound produced by Fast-Lip. We find that in all examples considered, after 5 minutes, GeoCert is able to generate larger lower-bounds compared to MIP. Table \ref{table:full} demonstrates these results for 5 randomly chosen examples.
%\textcolor{blue}{Admittedly, a more fine-grained update procedure for MIP might allow for faster verification up to a provided $\epsilon$, but MIP lacks the ability to reuse computation.} 

\begin{figure}
    \centering
    \includegraphics[scale=0.4]{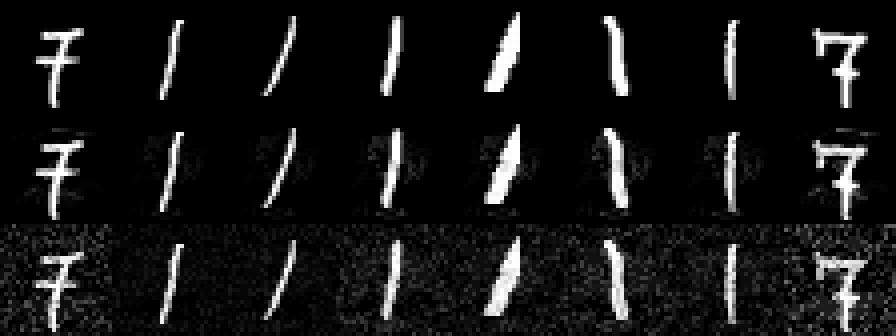}
    \caption{Original MNIST images (top) compared to their minimal distance adversarial examples as found by GeoCert (middle) and the minimal distortion adversarial attacks found by Carlini-Wagner $\ell_2$ attack. The average $\ell_2$ distortion found by GeoCert is 31.6\% less that found by Carlini-Wagner.}
    \label{fig:mnist_digits}
\end{figure}
% \begin{table}[]
% \centering 
% \begin{tabular}{@{}l|lllll@{}}
% \toprule
%          & \multicolumn{5}{c}{Example number}     \\ \midrule
% Method   & 1     & 2     & 3     & 4      & 5     \\ \midrule
% Fast-Lip & 1.782 & 1.319 & 1.501 & 1.9751 & 1.871 \\
% GeoCert  & 2.251 & 1.356 & 1.620 & 2.499  & 2.402 \\
% MIP      & 2.0   & 1.0   & 1.0   & 2.0    & 2.0   \\ \bottomrule
% \end{tabular}
% \end{table}

%% file: tex_files/07_conclusion.tex
\section{Conclusion}

This paper presents a novel approach towards both
bounding and exactly computing the pointwise robustness of piecewise linear neural networks for all convex $\ell_p$ norms. 
Our technique differs fundamentally from existing complete verifiers in that it leverages local geometric information to continually tighten a provable lower bound. Our technique is built upon the notion of computing the centered Chebyshev ball inside a polyhedral complex. We demonstrate that polyhedral complices have efficient boundary decompositions and that each decision region of a piecewise linear neural network forms such a polyhedral complex. We leverage the Lipschitz continuity of PLNN's to immediately output a nontrivial lower bound to the pointwise robustness and improve this lower bound until it ultimately becomes tight.

We observe that mixed integer programming approaches are typically faster in computing the exact pointwise robustness compared to our method. However,  our method provides intermediate valid lower bounds that are produced significantly faster. 
Hence, under a time constraint, our approach is able to produce distance lower bounds that are typically tighter compared to incomplete verifiers 
and faster compared to MIP solvers. 
An important direction for future work would be to optimize our implementation so that we can scale our method to larger networks. This is a critical challenge for all machine learning verification methods. 

%Ultimately, our paper approaches the problem of verifying pointwise robustness from a purely geometric perspective. This viewpoint is entirely novel and holds merit in the insight it provides. From this lens, verifying robustness for PLNNs is abstracted to finding the largest $\ell_p$ ball within a polyhedral complex. We show that such complices possess unique structure which leads to an efficient algorithm in Geocert. Although exact verification is provably difficult in the worst case, we show that we are competitive with other exact verifiers and additionally provide benefit as a compromise between complete and incomplete methods. Lastly we present a number of geometrically inspired speedups and provide evidence that the complexity of the linear regions of a PLNN is well removed from worst case.  

%% file: supplementary_arxiv.tex
\appendix
\input{supplementary_files/01_centered_chebyshev_discussion.tex}

\input{supplementary_files/02_boundary_decomp.tex}

\input{supplementary_files/03_potential_proofs.tex}

\input{supplementary_files/04_polyhedral_complices.tex}

\input{supplementary_files/05_plnn_proofs.tex}
\input{supplementary_files/06_upper_bounds.tex}

\input{supplementary_files/07_extra_experiments.tex}

%% file: supplementary_files/01_centered_chebyshev_discussion.tex
\section{Further discussion on Centered Chebyshev Balls}\label{app:chebyshev}
\subsection{Centered Chebyshev Ball of a Single Polytope} 
Here we present a more thorough discussion of the case of computing a centered Chebyshev ball for a single polytope, as well as general formulations for projections onto polytopes under various $\ell_p$ norms. 

Consider a polytope $\mathcal{P}: = \{x \; | \; Ax \leq b\}$. The problem of finding the the centered Chebyshev ball under an $\ell_p$ norm can written as the following optimization problem:
\begin{align} \label{eq:chebyshev-center}
\max &\quad t  \\
    \text{s.t.} \quad&  \sup_{||v|| \leq 1} a_i^T(x_0 + tv) \leq b_i \; \; \; \forall i\in [m]. \notag \\ 
    \notag
\end{align}

As a brief aside, note that if the center $x_0$ is not fixed, it is introduced as a variable in the optimization, and in general this requires a linear program to be solved. With a fixed center, each constraint can be rewritten as $t||a_i||_* \leq b_i - a_i^Tx_0$, for $||\cdot||_*$ being the dual norm of  $||\cdot||$. Thus the program becomes 
\begin{align} \label{eq:chebyshev-center-simple}
\max &\quad t  \\
    \text{s.t.} \quad  &t  \leq \frac{b_i - a_i^Tx_0}{||a_i||_*} \; \; \; \forall i \in [m]\notag \\ 
    \notag
\end{align}

% \begin{figure*}[h!]\label{fig:polytope-examples}
%     \centering
%     \subfigure{\includesvg[width=0.6\columnwidth,height=0.4\textheight]{images/svgs/arbit_union.svg}}
%     \hfill 
%     \subfigure{\includegraphics[width=0.6\columnwidth,height=0.6\textheight,             keepaspectratio]{images/glued_polytopes.pdf}}
%     \hfill
%     \subfigure{\includegraphics[width=0.6\columnwidth,height=0.6\textheight,keepaspectratio]{images/perf_glued_polytopes.pdf}}
%     \caption{Three examples of polytopes with nonempty intersection. \emph{(Left)} is neither glued nor perfectly glued, \emph{(Middle)} is glued but not perfectly glued, \emph{(Right)} is perfectly glued}
% \end{figure*} 

\begin{figure}
    \centering
    \includegraphics[scale=0.66]{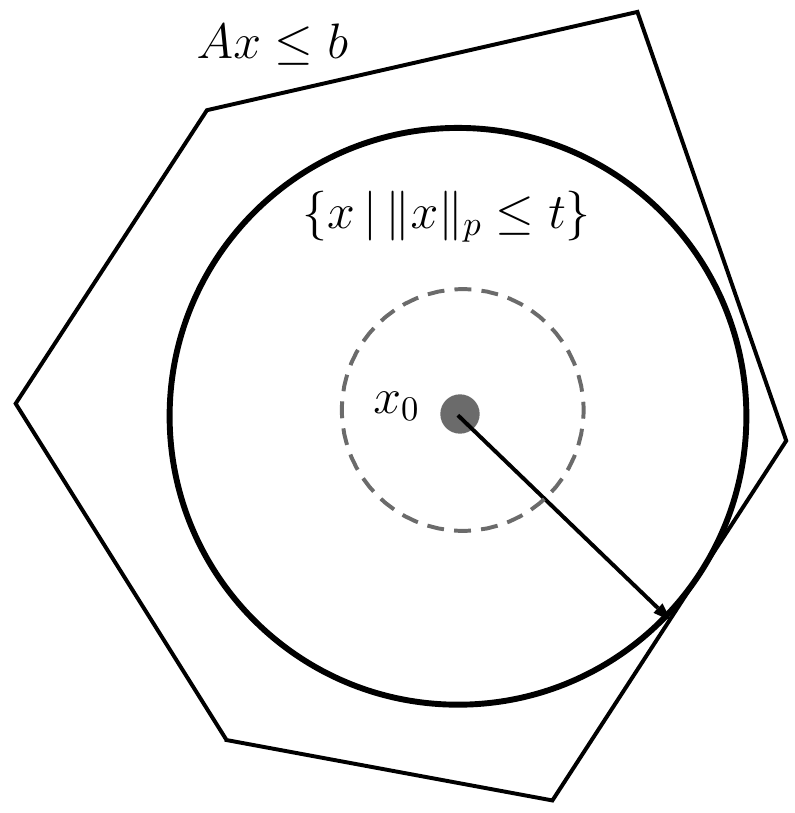}
    \includegraphics[scale=0.66]{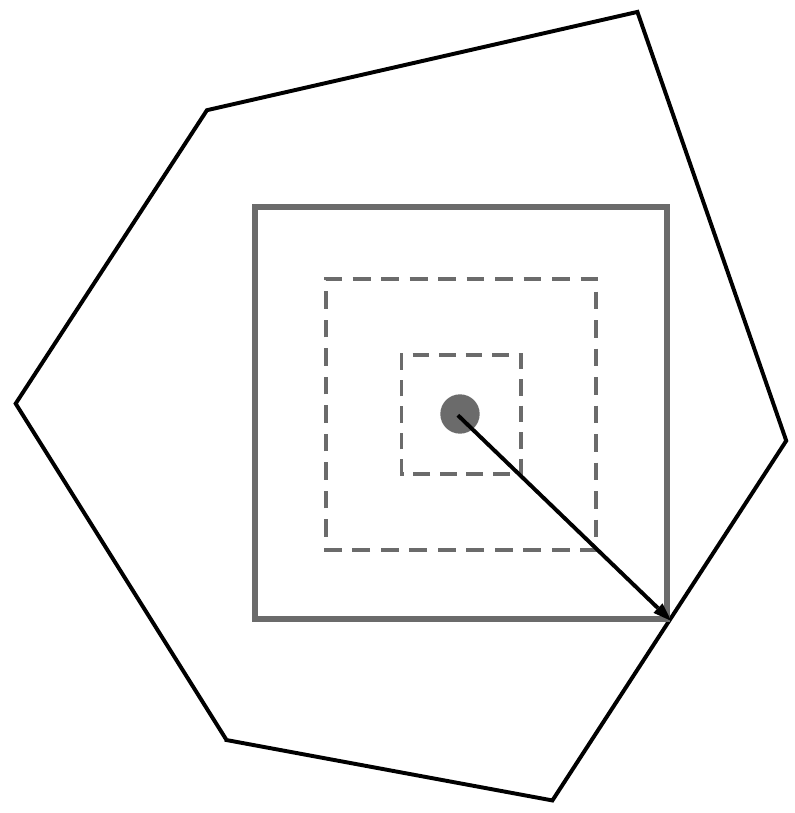}
    \caption{Pictorial examples of computing the centered Chebyshev ball for the $\ell_2, \ell_\infty$ norms.}
    \label{fig:chebyshev-single}
\end{figure}
which can be solved as taking the minimum over all $i$ of $\frac{b_i-a_i^Tx_0}{||a_i||_*}$. Understanding what is occurring here will be central to our theorems, so we decompose the above problem. Note that each constraint $a_i^Tx \leq b_i$ defines a hyperplane, and $\frac{b_i-a_i^Tx_0}{||a_i||_*}$ denotes the $\ell_p$ distance from $x_0$ to that hyperplane. In other words, this provides a lower bound on the $\ell_p$ distance to the facet of $\mathcal{P}$ generated by constraint $i$ being tight. However, the minimum of these lower bounds must be tight for the constraint that bounds the centered Chebyshev ball and therefore it suffices to compute this lower bound everywhere. Finding the centered Chebyshev ball is equivalent to finding the minimum distance to each component of the boundary of $\mathcal{P}$. An alternative, albeit more laborious, solution to finding the centered Chebyshev ball is to consider the minimal $\ell_p$ distance to $\delta \mathcal{P}$ directly by computing the $\ell_p$ distance to each facet of $\mathcal{P}$ and taking the minimum. 

\subsection{Projections onto Polytopes}
As our algorithm heavily relies on the ability to efficiently compute the projection to a facet, which is itself a polytope, we describe the general formulation here. Formally, provided a polytope $\mathcal{P}:=\{x \; | \; Ax \leq b\}$ and a point $x_0\notin \mathcal{P}$, we wish to compute $\min_{x \in \mathcal{P}} ||x_0 - x||_p$. To compute this exactly, we decompose $x$ in the minimum to $x_0 + v$ and optimize over $v$. This is a linear program in the $\ell_1$ case, and a linearly constrained quadratic program in the $\ell_2$ case. For $\ell_\infty$ we introduce $n+1$ auxiliary variables and $2n$ additional constraints: 
\begin{align} \label{eq:ellinf-lp}
\min_{t, v} \quad t & \\
    \text{s.t.} \quad A(x_0 + v) \leq b \notag \\ 
    t & \geq 0 \notag \\ 
    -t \cdot \mathbbm{1} & \leq v \leq t \cdot \mathbbm{1} \notag \\ 
\end{align}

In the $\ell_1$ case, we require $2n$ auxiliary variables:
\begin{align} \label{eq:ellone-lp}
\min_{t, v} \quad \sum t_i & \\
    \text{s.t.} \quad A(x_0 + v) &\leq b \notag \\ 
    t & \geq 0 \notag \\ 
    -t_i &\leq v_i \leq t_i  \notag \;\;\;\; \forall i\in [n]\\ 
\end{align}

And in the case of the $\ell_2$-norm, the objective becomes quadratic while the constraints remain linear:
\begin{align} \label{eq:ell2-lcqp}
\min_{v} \quad \sum_i v_i^2 & \\
    \text{s.t.} \quad A(x_0 + v) &\leq b \notag \\ 
    \notag 
\end{align}

In both cases there exist polynomial time algorithms to solve these exactly and efficient implementations to solve these quickly in practice \cite{Karmarkar1984-ah, Ye1989-ri}.
Thus, we can solve the problem of finding the \emph{centered Chebyshev ball} of a single polytope by solving the minimum distance to each facet, each formulated as an efficient LP or QP. 

\subsection{Notes on Hyperplanes}
Additionally we mention some cheap tricks that are useful when the polytopes of interest are $(n-1)$-dimensional. This implies that they lie entirely in some $(n-1)$-dimensional affine subspace, say $\mathcal{P} \subseteq H$ for $H:=\{x\; | \; a^Tx=b\}$. To compute a lower-bound on the projection of $x_0$ onto $\mathcal{P}$, one can compute the projection of $x_0$ onto $H$, which can be done in linear time in the dimension:
\begin{align} \label{eq:hyperplane-proj}
\min_{t, v} \quad t & \\
    \text{s.t.} \quad a^t(x_0 + v) \leq b \notag \\ 
    ||v|| &=1
    \notag 
\end{align}
Reformulating the first constraint, one has $t=\frac{b-a^Tx_0}{a^Tv}$. This quantity is minimzed when $a^Tv$ is maximized, and $\max_{||v||=1} a^Tv$ is, by definition, the dual norm $||\cdot||_*$ of $a$. Hence the projection onto a hyperplane is $\frac{b-a^Tx_0}{||a||_*}$.

In section \ref{sec:speedups}, we mention that it is efficient to compute the feasibility of $H\cap B$ for $B$ being some hyperbox defined by coordinate lower and upper bound vectors, $l$ and $u$ as $\{x \; | \; l \leq x \leq u\}$. We can decompose $a$ into its nonnegative components $a^+$ and its negative components $a^-$ such that $H = \{x\;|\; (a^++a^-)^Tx = b\}$. Then, by interval arithmetic, we notice that the set $\{c \; | \; a^Tx \;\;\; \forall x \in B\}$ is the interval $[(a^+)^Tl + (a^-)^Tu, (a^-)^Tl + (a^+)^Tu]$. Iff $b$ is contained in this interval, then the intersection $H\cap B$ is nonempty.

%% file: supplementary_files/02_boundary_decomp.tex
\section{Proofs about Boundary Decompositions}\label{app:boundary}
Here we prove our theorems about efficient boundary decomposotions of polyhedral complices. First we state a hardness result that claims that for arbitrary nonconvex polytopes, the size of the smallest convex decomposition of the boundary may be exponential in the dimension. 

%%%%%%%%%%%%%%%%%%%%%%%%%%%%%%%%%%%%%%%%%%%%%%%%%%%%%%%%%%%%%%%
%                                                             %
%                       HYPERPLANE ARRANGEMENT                %
%                                                             %
%%%%%%%%%%%%%%%%%%%%%%%%%%%%%%%%%%%%%%%%%%%%%%%%%%%%%%%%%%%%%%%

\begin{theorem}\label{thm:hyp_arrang}
There exists a collection of polytopes $\mathscr{P}=\{\mathcal{P}_1, \dots \mathcal{P}_k\}$ each with dimension $n$ and 2 constraints (for a total of $2k$ constraints) such that the boundary of $\bigcup_{i\in [k]} \mathcal{P}_i$ is composed of $\Omega(k^{n-1})$ convex components.
\end{theorem}

\begin{figure}
    \centering
    \includegraphics[scale=0.5]{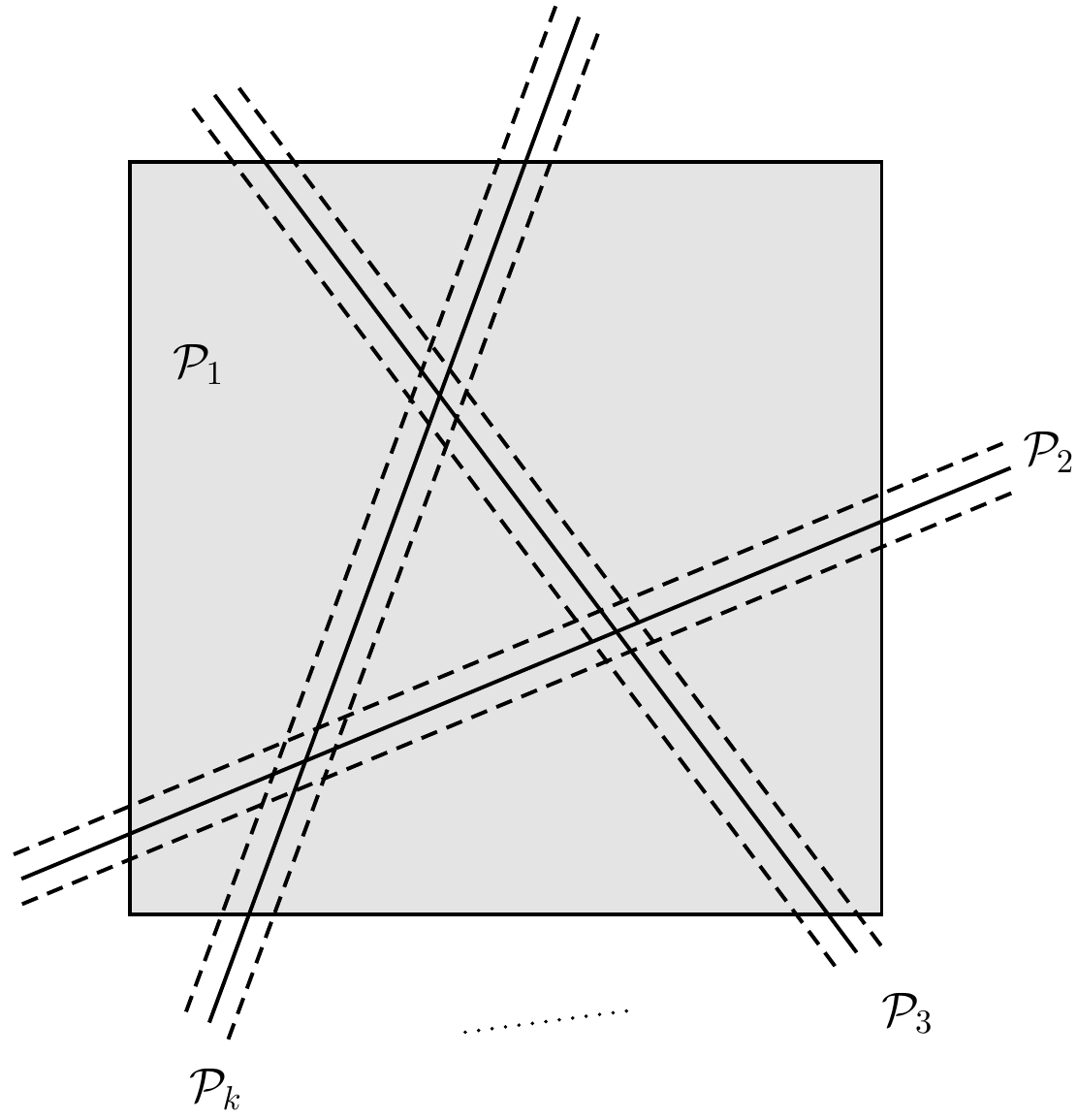}
    \caption{Pictorial aid for Theorem \ref{thm:hyp_arrang}}
    \label{fig:hyperplane-arrangement}
\end{figure}
\begin{proof}
We prove this by construction. We rely crucially on a result from hyperplane arrangements. It is a classical result that given a choice in placement of $m$ hyperplanes in $\mathbb{R}^n$, 
the maximum number of regions that can be generated is given, in closed form as $R(n, m) := 1 + \sum_{j=1}^n \binom{m}{j}$ \cite{Stanley2004-qv}.
Leveraging this, we construct our polytopes. Let $\mathcal{P}_1 =\{x \; | \; 0 \leq x_1 \leq 1\}$ such that it has exactly two facets, where each facet is an $(n-1)$ flat. Let $\mathcal{A}$ be an arrangement of $k-1$ hyperplanes in $\mathbb{R}^{n-1}$ that generates a maximal number of regions. Each one of the regions generated by $\mathcal{A}$ is certainly a polytope contained in $\mathbb{R}^{n-1}$, so since there are finitely many polytopes each with finitely many vertices, let $\epsilon$ be the minimal distance between any two vertices within the same polytope. Let the $i^{th}$ hyperplane in $\mathcal{A}$ be defined as $\{x \in \mathbb{R}^{n-1} \; | \; a_i^Tx = b_i\}$. Then we can define $\mathcal{P}_{i+1} := \{x\in \mathbb{R}^n \; | \; b_i - \epsilon/3 \leq (0, a_i)^T x \leq b_i+\epsilon/3\}$. Thus the $(n-1)$-flat that describes each facet of $\mathcal{P}_1$ remains broken up into $R(n-1,k-1) = \Omega(k^{n-1})$ disjoint convex components. Each of these exists on the boundary of the union of $\mathscr{P}$.
\end{proof}

% \begin{figure}[htbp]
%     \centering
%     \includesvg[width=0.8\linewidth,height=0.8\textheight]{images/svgs/hyperplane_arrangement.svg}
%     \caption{Pictorial reference for Theorem \ref{thm:hyp_arrang}. Hyperplanes in $\mathbb{R}^{n-1}$ are maximally arranged to create $\Omega(k^{n-1})$ partioned regions of a $(n-1)$-flat of polytope $\mathcal{P}_1$. These hyperplanes are then each given a small amount of freedom in the remaining dimension of $\mathcal{R}^n$ s.t. all regions are preserved.}
% \end{figure}

%%%%%%%%%%%%%%%%%%%%%%%%%%%%%%%%%%%%%%%%%%%%%%%%%%%%%%%%%%%%%%%
%                                                             %
%                       PC BOUNDARY DECOMPOSITION             %
%                                                             %
%%%%%%%%%%%%%%%%%%%%%%%%%%%%%%%%%%%%%%%%%%%%%%%%%%%%%%%%%%%%%%%

Now we can restate and prove our theorems regarding the efficient boundary decompositions of polyhedral complices. 

\begin{customthm}{\ref{thm:main_algo}}
Given a polyhedral complex, $\mathscr{P} =\{\mathcal{P}_1, \dots \mathcal{P}_k\}$, where $\mathcal{P}_i$ is defined as the intersection of $m_i$ closed halfspaces. Let $M=\sum_i m_i$, and let $x_0$ be a point contained by at least one such $\mathcal{P}_i$. Then the boundary of $\bigcup_{i \in [k]} \mathcal{P}_i$ is represented by at most $M$ $(n-1)$-dimensional polytopes. There exists an algorithm that can compute this boundary in $\mathcal{O}(poly(n, M, k))$ time. 
\end{customthm}

\begin{proof} 
Let $Z=\bigcup_{i\in [k]} \mathcal{P}_i$. Let $F_{i,j}$ refer to the $j^{th}$ facet of $\mathcal{P}_i$, and let $\mathcal{F}_i$ be the set of facets of $\mathcal{P}_i$ that are not facets of any other $\mathcal{P}_j$. Then, letting $T = \bigcup_{i \in [k]} \mathcal{F}_i$. We claim that the boundary of $Z$ is exactly $T$. 

Without loss of generality, assume that $Z$ is a single connected component, in the topological sense. If $Z$ were multiple connected components, then we could handle each of them in turn. To demonstrate that $T$ is the boundary of $Z$ we need to show that for any $x\in T$ that points \ref{item:inbound}, \ref{item:outbound} of definition \ref{def:boundary} hold, and that condition \ref{item:outbound} fails for any point $y \in Z \setminus T$. 

To demonstrate point \ref{item:inbound} above, we note that $x\in\mathcal{P}_i$ for at least one $\mathcal{P}_i$. By assumption each $\mathcal{P}_i$ has a nonempty interior, and thus contains some point $y\in \mathcal{P}_i$ for which a neighborhood $N(y) \subset \mathcal{P}_i$. Thus if $\mathcal{P}_i$ is given as an $H$-polytope of the form $\{x \; | \; Ax \leq b\}$, then $Ay < b$. Since $\mathcal{P}_i$ is convex, then any convex combination between $x,y$ is contained in $\mathcal{P}_i$, and in fact for all $\lambda \in [0, 1)$, $A(\lambda x + (1-\lambda) y) < b$. Certainly any point $z$ such that $Az<b$ has a neighborhood $N(z)$ contained in $\mathcal{P}_1$.

Proving that $x\in T$ satisfies point \ref{item:outbound} is more complicated. Let $\mathcal{Q}$ be a facet containing $x$, and let $\mathcal{P}_i$ be a polytope containing $\mathcal{Q}$. Let $H$ be the hyperplane containing $\mathcal{Q}$. Then for all $j \neq i$, $\mathcal{P}_i \cap \mathcal{P}_j$ is either the empty set or resides in a face of $\mathcal{P}_j$ of dimension at most $(n-2)$. A standard result about polytopes states that if $\mathcal{Q}$ is an $(n-1)$ dimensional polytope, it can be defined by the set $\{x \; | \; A^=x=b^= \land A^-x\leq b^-\}$ where $A^=$ has rank 1. Additionally there exists a point $y\in \mathcal{Q}$ such that $A^-y < b$ \cite{Schrijver1998-em}. Then every point along the open line segment $(x, y)$ is contained in the relative interior of $\mathcal{Q}$, and by definition cannot be contained in any face of $\mathcal{P}_j$ for $j\neq i$. Further, since the relative interior of $\mathcal{Q}$ is open, every point $w$ along $(x,y)$ is contained in a neighborhood $N(w)$, with restriction to $H$ $N(w)_{|H}$. Then certainly $N(w)_{|H} \subseteq relInt(\mathcal{Q}) \subset \mathcal{Q}$ , which implies that $N(w)_{|H}$ is disjoint from $\cup_{j\neq i}\mathcal{P}_j$. 

Let $H^-$ be the closed halfspace defined by $H$ containing $\mathcal{P}_i$, then $N(w)\cap (H^-)^c$ is both open and disjoint from $\mathcal{P}_i$ in addition to being disjoint from $\mathcal{P}_j$ for all $j\neq i$. Let $c$ be a point in $N(w)\cap Z^c$, such that the open line segment between $(w,c)$ is contained in $N(w)\cap Z^c$. We now restrict our attention to the 2-dimensional linear subspace of $\mathbb{R}^n$ containing $x,w,c$, denoted as $V$. Each $\mathcal{P}_{j|V}$ is either the emptyset or a polytope containing $x$. Let $\mathcal{U}_{j|V}$ be the set of these 2-d restricted polytopes containing $x$, and note that each $\mathcal{U}_{j|V}$ intersects with $\mathcal{P}_{i|V}$ only at $x$. Because each element of $\mathcal{U}_{|V}$ intersects with $\mathcal{P}_{i|V}$ only at $x$, there must a hyperplane $H_j$, (line in $V$) passing through $x$ separating each element of $\mathcal{U}_{|V}$ and $c$. Let $H_j^+$ be the closed halfspace defined by $H_j$ containing $c$. Then $\cap H_j$ defines a polytope $\mathcal{S}$ that only intersects with $\mathcal{P_i}$ at $x$. The line segment between $(x,c)$ lies inside $\mathcal{S}$ and thus does not intersect any $\mathcal{P}_{j|V}$ for $j \neq i$. $(x,c)$ also lies strictly on one side of the hyperplane $H$ that $\mathcal{Q}$ resides in, and thus every point along $(x,c)$ is not contained in $\mathcal{P}_i$. Hence, $(x, c)$ is not contained in $Z$, as desired. 

Finally, to show that there is no point $y$ in the boundary of $Z$ that not contained in $T$. It suffices to show that $Z\setminus T$ is open, as if this were the case, then any $y\in Z\setminus T$ would be contained in a neighborhood $N(y) \subseteq Z\setminus T$ and thus fail to meet condition \ref{item:outbound} of the definition of the boundary. Let $x\in Z\setminus T$. Then $x$ is contained in the interior of some $\mathcal{P}_i$ or it is contained in a facet contained in both $\mathcal{P}_i, \mathcal{P}_j$, for some $i,j$. This follows from the fact that $x$ either is contained in a facet of some $\mathcal{P}_i$ or not. If not, $x$ is strictly in the interior of some $\mathcal{P}_i$ and is contained in a neighborhood $N(x) \subset \mathcal{P}_i$. If so, then $x$ needs to be contained in a facet, $F_{i,j}$ of $\mathcal{P}_i$ and $\mathcal{P}_j$, else $x\in T$. Either $x$ is contained in the relative interior of $F_{i,j}$ or not. If so, then a neighborhood of $x$, $N(x)$, is bisected by $F_{i,j}$, where each half is contained in either $\mathcal{P}_i$ or $\mathcal{P}_j$. If not, then $x$ needs to be contained in a facet of some $\mathcal{P}_m$, for $m\neq i,j$, because it needs to be contained in some other facet of $\mathcal{P}_i$. This other facet needs to be a facet of some $\mathcal{P}_m$ because otherwise it would be contained in $T$ and certainly $\mathcal{P}_i\cap\mathcal{P}_j = F_{i,j}$ such that $m\neq j$. We repeat this process until we have enumerated all facets containing $x$, of which there are at most $\binom{k}{2}$. There are then at most $k$ polytopes containing $N(x)$, and their union contains $N(x)$. Thus $Z\setminus T$ is open.

To demonstrate that $T$ is represented by at most $M$ polytopes and that $T$ can be computed in $\mathcal{O}(poly(n, M, k))$ time, note that each polytope $\mathcal{P}_i$ has at most $m_i$ facets, and not all of these are included in $T$. Thus the number of facets, and hence polytopes, that define $T$ is at most $\sum m_i= M$. Enumerating each of these polytopes can be done in time linear in $M$. To compare if two facets are equivalent, one can find a point $y \in F_{i,j}$ such that it is in the relative interior of $F_{i,j}$. Such a point can be found in polynomial time using a linear program. Since $\mathscr{P}$ is a polyhedral complex, if such a $y$ is contained in $F_{i,j}$ and $F_{i', j'}$ then $F_{i,j}=F_{i',j'}$. There are at most $\binom{M}{2}$ facets, so $T$ can be determined in time polynomial in $n, M, k$.

\end{proof}

%% file: supplementary_files/03_potential_proofs.tex
\section{Proofs of Correctness for GeoCert}\label{app:potential_proofs}

In this section we expand upon the graph theoretic interpretation of GeoCert and prove its correctness. Recall the setup: given a polyhedral complex $\mathscr{P}$, which can be viewed as a bipartite graph of $n$-dimensional polytopes and their $(n-1)$-dimensional faces, some of which are labeled as `boundary' facets, our goal is to return the boundary facet which admits minimal distance to a fixed point $x_0$. In our primary discussion we replaced 'distance' with a `potential' function. Formally, we let our pointwise potential to be some function $\phi: \mathbb{R}^n \rightarrow \mathbb{R}$, and the facetwise potential, $\Phi: \mathcal{P}(\mathbb{R}^n) \rightarrow (\mathbb{R} \cup \{+\infty\})$ to be defined as 
\begin{equation}
    \Phi(\mathcal{F}) = \begin{cases} 
        +\infty, & \text{if } \mathcal{F} = \emptyset \\ 
        \min\limits_{y\in \mathcal{F}} \phi(y), & \text{otherwise}
    \end{cases}
\end{equation}

Certainly, letting $\phi(y):= ||y - x_0||$ and finding the boundary facet with minimal potential $\Phi$ is equivalent to finding the facet with minimal distance to $x_0$. However, this choice of $\phi$ is not the only valid one for which GeoCert will provide the corect answer to the centered Chebyshev ball problem. To this end, we provide a sufficient condition on a pointwise potential function $\phi$ such that GeoCert will still provide the correct answer. We can then demonstrate that any potential function satisfying this property will cause GeoCert to return the correct answer. Finally we can show that the $\ell_p$-distance potential satisfies these properties, and that the lipschitz potential described in Section \ref{sec:speedups} also satisfies this property. 

\begin{definition}\label{def:ray-monotonic}
    Given a potential function $\phi$ defined only on the set of points contained in a polyhedral complex $\mathscr{P}$, we let $\eta_v(t) := \phi(x_0 + t\cdot v)$ for any vector $v$ and any positive scalar $t > 0$. Then we say that $\phi$ is \textbf{ray monotonic} if for every $v, t > 0$, $\dfrac{\delta \eta}{\delta t}(t) \geq 0$. 
\end{definition}

With this definition in hand, we can prove a structural invariant of the operation of GeoCert that will directly prove the claim of correctness. 
\begin{lemma}\label{lemma:pq-monotonic}
    For any polyhedral complex $\mathscr{P}$ point $x_0$, and ray-monotonic potential $\phi$, let $\mathcal{F}_i$ be the facet popped at the $i^{th}$ iteration of GeoCert. Then for all $i < j$, $\Phi(\mathcal{F}_i) \leq \Phi(\mathcal{F}_j)$. 
\end{lemma}
\begin{proof}
We proceed by induction. In the base case we only consider the first and second iteration. Supposing without loss of generality that $x_0$ is contained in exactly one polytope $\mathcal{P}\in \mathscr{P}$. Then the initial set of facets added to the priority queue is exactly the set of facets of $\mathcal{P}$, which we denote as $\{\mathcal{F}_\mathcal{P}(1), \mathcal{F}_\mathcal{P}(2), \dots, \mathcal{F}_\mathcal{P}(k)\}$ which are ordered by potential, without loss of generality. 

At the first iteration, $\mathcal{F}_\mathcal{P}(1)$ is popped, and a new polytope $\mathcal{S}$ is opened. The set of facets of added to the priority queue $Q$, also ordered by potential, is $\{\mathcal{F}_\mathcal{S}(1), \mathcal{F}_\mathcal{S}(2), \dots, \mathcal{F}_\mathcal{S}(k)\}$. We would like to show that whichever facet $\mathcal{F}_2$, is popped at iteration 2 must have that $\Phi(\mathcal{F}_2) \geq \Phi(\mathcal{F}_\mathcal{P}(1))$. As, by definition, for all $i > 1$,  $\Phi(\mathcal{F}_\mathcal{P}(1)) \leq \Phi(\mathcal{F}_\mathcal{P}(i))$ it suffices to show that any facet $\mathcal{F}_\mathcal{S}$ of $\mathcal{S}$ added to the priority queue must have $\Phi(\mathcal{F}_\mathcal{P}(1)) \leq \Phi(\mathcal{F}_\mathcal{S})$. For any facet of $\mathcal{F}_\mathcal{S}$, we have that $\Phi(\mathcal{F}_\mathcal{S}) := \min\limits_{y \in \mathcal{F}_\mathcal{S}(1))} \phi(y)$. Letting $y_{min}$ be an element of the argmin of this minimum, we utilize the ray-monotonic property of $\phi$. We let $v = y_{min} - x_0$ and note that $\Phi(\mathcal{F}_\mathcal{S})=\phi(x_0 + v)$. As $y_{min}$ is not contained in the interior of $\mathcal{P}$, there must exist some $t \in [0, 1]$ such that $x_0 + tv$ lies in a facet of $\mathcal{P}$. By definition $\Phi(\mathcal{F}_\mathcal{P}(1)) \leq \phi(x_0 + tv) \leq \phi(x_0 + v)$, where the first inequality comes from the definition of $\Phi$, and the second inequality comes from the ray-monotonicity of $\phi$. This concludes the base case. 

The inductive step follows by a similar argument. Suppose the claim holds up to iteration $i-1$. At the $i^{th}$ iteration we pop facet $\mathcal{F}_i$, open up a previously-unseen polytope $\mathcal{S}$, and add a set of facets each corresponding to another unseen polytope: hence no potential facet added has been previously added to the priority queue. Again, considering any new facet $\mathcal{F}_\mathcal{S}$ and the argmin of its potential 
$$y_{min}\in \argmin\limits_{y \in \mathcal{F}_\mathcal{S}} \phi(y)$$
we note that $y_{min}$ is not contained in the interior of any of the set of seen polytopes $C$. Then again letting $y_{min} = x_0 + v$, there exists some $t\in (0, 1]$ such that $x_0+tv$ lies in some facet $\mathcal{G}$ that is contained in the priority queue at iteration $(i-1)$. Since $\Phi(\mathcal{F}_{(i-1)}) \leq \Phi(\mathcal{G}) \leq \phi(y_{min}) = \Phi(\mathcal{F}_\mathcal{S})$, we maintain our structural invariant and the proof is complete.
\end{proof}

\begin{theorem}\label{thm:geocert-correct}
    For a fixed polyhedral complex $\mathscr{P}$, a fixed input point $x_0$ and a potential function $\phi$ that is ray-monotonic, GeoCert returns a boundary facet with minimal potential $\Phi$. 
\end{theorem}
\begin{proof}
Leveraging Lemma \ref{lemma:pq-monotonic}, we note that since we only pop facets in non-decreasing order, the first `boundary facet' that is popped will be a boundary facet with minimal potential.
\end{proof}

Now we simply need to show that both choices of potential function discussed satisfy the ray-monotonicity property. 
\begin{corollary}\label{cor:lp-potential}
The distance potential, $\phi_{lp}(y):=||y - x_0||$ satisfies ray-monotonicity and Geocert using this as a potential returns the minimal distance boundary facet.
\end{corollary}
\begin{proof}
We fix a vector $v$ and any scalar $t>0$. We define 
\begin{equation}
\eta_v(t):= ||(x_0+tv)-x_0|| = |t|\cdot ||v|| = t \cdot||v|| 
\end{equation}
Then $\dfrac{\delta \eta_v}{\delta t} = ||v|| \geq 0$ for all $t >0, v$. 
\end{proof}

\begin{corollary}
For a PLNN $f :\mathbb{R}^n \rightarrow \mathbb{R}^k$ and a point $x_0$, let $i := \argmax_j f_j(x_0)$. Let $DR(x_0) = \{x \; | \; \argmax_j f(z) =i\}$. Define $g_j(x) = f_i(x) - f_j(x)$ for all $j\neq i$, and let $L_j$ be a bound on the $\ell_q$ lipschitz constant of $g_j$: 
\begin{equation}
    |g_j(x) - g_j(y)| \leq L_j ||x - y||_p \; \; \; \;\;\;\forall x, y \in DR(x_0)
\end{equation}
then the potential  
\begin{align}
    \phi_{lip,j}(y) &:= ||y-x_0||_p + \frac{g_j(y)}{L_j}\\
    \phi_{lip}(y) &:= \min_{j} \phi_{lip,j}(y)
\end{align}

satisfies ray-monotonicity and Geocert using this as a potential returns the minimal distance boundary facet.
\end{corollary}
\begin{proof}
We prove the ray-monotonicity for each $\phi_j$ and then demonstrate that this holds for their minimum as well. First we note that for every point $x \in DR(x_0)$ has that $g_j(x) \geq 0$. Fixing some $\phi_j$, $v$, and $t >0$ such that $x_0+tv \in DR(x_0)$, we consider 

\begin{equation}
    \eta_{j,v}(t):= \phi_{lip,j}(x_0 + tv) = t||v||_p + \frac{g_j(x_0 + tv) - g_j(x_0)}{L_j}
\end{equation}
which has derivative 
\begin{align}
    \dfrac{\delta \eta_{j,v}}{\delta t}(x_0 + tv) &= ||v||_p + \frac{1}{L_j}\dfrac{\delta g_j}{\delta t}(x_0 + tv) \\
    &= ||v||_p + \frac{1}{L_j}\langle v, \nabla g_j(x_0 + tv)\rangle \\ 
    &\geq ||v||_p  -\frac{1}{L_j} ||V||_p| ||\nabla g_j(x_0 + tv)||_q\\ 
    &\geq ||v||_p(1 - 1) \\
    &\geq 0
\end{align}
Where the first inequality comes from Hölder's inequality, and the second inequality comes from the fact that the norm of the gradient is bounded by the lipschitz constant. And since the minimum of monotonically increasing functions is also monotonically increasing, $\phi$ is ray-monotonic. This implies that GeoCert returns the minimal potential facet. However, note that along any boundary facet $\mathcal{F}_{bound}$, there exists a $j$ such that $g_j(y)=0 \forall y\in \mathcal{F}_{bound}$. Since each $g_j(y)\geq 0$ for all $y\in DR(x_0)$ for any $y \in \mathcal{F}_{bound}$, $\phi(y)=||x_0-y||_p$. In other words, this potential function is equivalent to the $\ell_p$ potential along  the decision boundary. Hence the first `boundary facet' popped is the boundary facet with minimal $\ell_p$ distance, as desired.
\end{proof}

\paragraph{Remarks:} Recall that as a subroutine, GeoCert using $\phi_{lip}$ as a potential, must compute $\Phi_{lip}(\mathcal{F})$ for each possible facet $\mathcal{F}$ to be added to the priority queue. This amounts to solving the following optimization problem 
\begin{equation}
    \Phi_{lip}(\mathcal{F}) := \min\limits_{y\in \mathcal{F}}\Big( ||y - x_0||_p + \min_{j\neq i} \frac{g_j(y)}{L_j}\Big)
\end{equation}
Along each piecewise linear region of a PLNN, certainly $f$ is a linear function, as is $g_j$. Hence, computing the minimum of $\phi_{lip, j}$ across a facet requires as much computation time as computing the $\ell_p$ projection to a facet. Since $\min_{j\neq i} g_j(y)$ is a pointwise minimum and hence not convex, computing $\Phi_{lip}$ is no longer computable by a single convex program. However one can minimize this for each $\phi_{lip, j}$ and return the overall minimum. This now requires multiple convex programs per facet. We find that (i) using a warm-start for our optimizations allows the second-through-final to finish much more quickly than the initial optimization, and (ii) a variant of GeoCert can be used where the facet-wise potential is replaced with a polytope-wise potential. Under this formulation, the number of optimizations per polytope with $m$ constraints goes from $m$, in the case of the $\ell_p$ potential, to $m + (k-1)$ where $k$ is the number of logits: we simply need to compute the feasibility of each facet ($m$ linear programs), to determine the neighbors of the right vertices in the graph, and $(k-1)$ optimizations to compute the polytope-wise potential.

Finally, we remark about the efficient computation of $L_j$. Under a fixed domain $\mathcal{D}$, if a lower and upper bound to each input to each ReLU of the neural net is known, a nontrivial upper bound to each $L_j$ can be computed with as much computation as is required by eight forward passes through the PLNN \cite{fastlin}. Indeed, by leveraging $\phi_{lip}$ as a potential, one can effectively propagate the lower-bound to pointwise robustness as computed by Fast-Lip: instead of computing a certifiable lower bound only on $f$ evaluated at $x_0$, as Fast-Lip does, the certifiable lower bound is now computed across every facet in the `frontier set' which expands outwards as GeoCert runs. This allows for Fast-Lip to converted into continually increasing lower bound.

%% file: supplementary_files/04_polyhedral_complices.tex
\section{Polyhedral Complex Properties}\label{app:polyhedral_complices}
Here we will restate and prove the lemmas regarding iterative construction of polyhedral complices, and other useful tools when considering the centered Chebyshev ball contained in a polyhedral complex. 

\begin{customlem}{\ref{lemma:hyperplane-glue}}
Given an arbitrary polytope $\mathcal{P}:= \{x \; | \; Ax \leq b\}$ and a hyperplane $\mathcal{H}:= \{x \; | \; c^Tx = d\}$ that intersects the interior of $\mathcal{P}$, the two polytopes formed by the intersection of $\mathcal{P}$ and the each of closed halfpsaces defined by $\mathcal{H}$ are PC.
\end{customlem}
\begin{proof}
Let $\mathcal{H}^+:= \{x \; | \; c^Tx \geq d\}$ and $\mathcal{H}^- := \{x \; | \; c^Tx \leq d\}$, with $\mathcal{P}^+ := \mathcal{P}\cap \mathcal{H}^+$ and $\mathcal{P}^- := \mathcal{P} \cap \mathcal{H}^-$. Then each of $\mathcal{P}^+$, $\mathcal{P}^-$ are polytopes with nonempty interior and their intersection is exactly $\mathcal{P}\cap \mathcal{H}$, which is a face of both $\mathcal{P}^+, \mathcal{P}^-$.
\end{proof}

\begin{customlem}{\ref{lemma:double-hyperplane-glued}}
Let $\mathcal{P},\mathcal{Q}$ be two PC polytopes and let $H_\mathcal{P}$, $H_\mathcal{Q}$ be two hyperplanes that define two closed halfspaces each, $H^+_\mathcal{P}, H^-_\mathcal{P}, H^+_\mathcal{Q}, H^-_\mathcal{Q}$. If $\mathcal{P}\cap\mathcal{Q}\cap H_\mathcal{P} =\mathcal{P}\cap\mathcal{Q}\cap H_\mathcal{Q}$ then the subset of the four resulting polytopes $\{\mathcal{P}\cap H^+_\mathcal{P}, \mathcal{P}\cap H^+_\mathcal{P}, \mathcal{Q}\cap H^+_\mathcal{Q}, \mathcal{Q}\cap H^+_\mathcal{Q}\}$ with nonempty interior forms a polyhedral complex.
\end{customlem}
\begin{figure}
    \centering
    \includegraphics[scale=0.5]{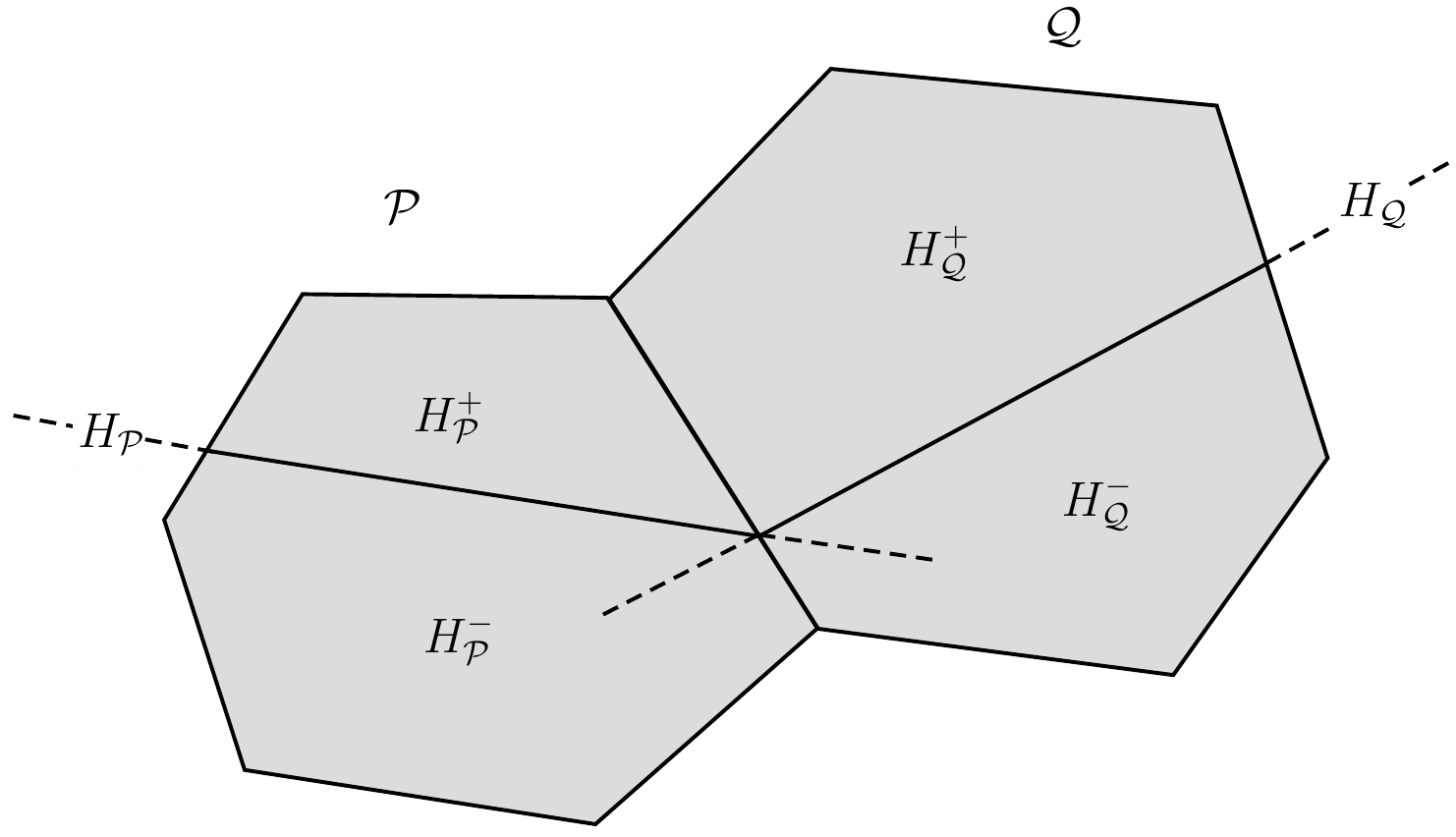}
    \caption{Pictorial aid for Lemma \ref{lemma:double-hyperplane-glued}.}
    \label{fig:double-hyperplane-glued}
\end{figure}

\begin{proof}
Let $F=\mathcal{P}\cap \mathcal{Q}$, which by definition is a face of both $\mathcal{P}, \mathcal{Q}$. Without loss of generality we can align the hyperplanes $H_\mathcal{P}, H_\mathcal{Q}$ such that $F\cap H^+_\mathcal{Q}=F\cap H^+_\mathcal{P}$. For ease of notation, we'll let $\mathcal{P}^+$ denote $\mathcal{P}\cap H^+_\mathcal{P}$, and similarly for $\mathcal{P}^-$, $\mathcal{Q}^+$, $\mathcal{Q}^-$. If $H_\mathcal{P}$ does not intersect the interior of $\mathcal{P}$, then exactly one of $\mathcal{P}^+, \mathcal{P}^-$ has empty interior and can be ignored. Otherwise, by lemma \ref{lemma:hyperplane-glue}, $\mathcal{P}^+$, $\mathcal{P}^-$ are PC, and likewise for $\mathcal{Q}^+$, $\mathcal{Q}^-$. To handle the cross-terms we proceed by cases. Letting $S =F \cap H_\mathcal{P}$, we handle the following four cases: (i) $S=\emptyset$, (ii) $S$ is a face of $F$, (iii) $S= F$, or (iv) none of the above. 

(i): In the case that $S=\emptyset$, then either $\mathcal{P}^+\cap F$ or $\mathcal{P}^-\cap F$ is empty. Likewise for $\mathcal{Q}^+\cap F, \mathcal{Q}^-\cap F$. Assume without loss of generality that $\mathcal{P}^+\cap F = \mathcal{Q}^+\cap F = \emptyset$. Then certainly $\mathcal{P}^+$ is disjoint from $\mathcal{Q}$ and therefore both $\mathcal{Q}^+$, $\mathcal{Q}^-$. Likewise for the interaction between $\mathcal{Q}^+$ and $\mathcal{P}^-, \mathcal{P}^+$. Finally, since $S=\emptyset$, $F$ is a face of both $\mathcal{P}^-$ and $\mathcal{Q}^-$ and $\mathcal{P}^-\cap \mathcal{Q}^-=F$, hence they are PC.

(ii): In the case that $S$ is a face of $F$, we label this face $G$. First note that $F$ needs to be fully contained by either $F\cap H^+_\mathcal{P}$ or $F\cap H^-_\mathcal{P}$. Thus $F$ is either a face of $\mathcal{P}^+$ or $\mathcal{P}^-$, where we can assume without loss of generality that it is a face of $\mathcal{P}^-$. Similarly, assume $F$ is a face of $\mathcal{Q}^-$, implying that $\mathcal{P}^-$ and $\mathcal{Q}^-$ are PC. By this assumption, $\mathcal{P}^+\cap F = G$. Note that $G$ is a face of $\mathcal{P}^+$.Since $G$ is a face of $F$, it is also a face of $\mathcal{Q}^-$, and $\mathcal{P}^+\cap\mathcal{Q}^-=G$, which is a face of each of them and therefore $\mathcal{P}^+$ and $\mathcal{Q}^-$ are PC. Likewise for $\mathcal{Q}^+$ and $\mathcal{P}^-$. Finally note that since $\mathcal{P}^+\cap F = \mathcal{Q}^+\cap F=G$, implying that $\mathcal{P}^+\cap \mathcal{Q}^+ = G$, hence $\mathcal{P}^+$ and $\mathcal{Q}^+$ are PC. 

(iii): If $S=F$, then we can assume without loss of generality that $\mathcal{P}^-=\mathcal{P}$ and $\mathcal{P}^+=F$, and similarly for $\mathcal{Q}$. Then since $\mathcal{Q}^+=\mathcal{P}^+=F$ they do not have nonempty interior and can be ignored. By definition $\mathcal{P}^-$ and $\mathcal{Q}^-$ are PC, and $\mathcal{P}^-, \mathcal{Q}^+$ are as well. 
(iv): In the final case, $S$ is neither the emptyset, $F$, nor a face of $F$. Then $F\cap H^+_\mathcal{Q}$ and $F\cap H^-_\mathcal{Q}$ are both nonempty polytopes with the same dimensionality as $F$. Letting $S^+ = F\cap H^+_\mathcal{Q}$, and defining $S^-$ likewise, note that $S$ is a face of $S^+, S^-$, by the same argument used in \ref{lemma:hyperplane-glue}. Since $F$ is a face of $\mathcal{P}$, $S^+$ is a face of $\mathcal{P}^+$ and likewise for $\mathcal{Q}^+$. And since $\mathcal{P}^+\subseteq \mathcal{P}$, $\mathcal{P}^+\cap \mathcal{Q}^+ \subseteq \mathcal{P}\cap \mathcal{Q} =F$. But $\mathcal{P}^+ \cap F= S^+$ and $\mathcal{Q}^+\cap F=S^+$, thus $\mathcal{P}^+\cap \mathcal{Q}^+=S^+$. Hence $\mathcal{P}^+$ and $\mathcal{Q}^+$ are PC. Likewise for $\mathcal{P}^-$ and $\mathcal{Q}^-$. Since $\mathcal{P}^+\cap \mathcal{Q}^- = S$ and $S$ is a face of $S^+$, $S^-$, it is a face of both $\mathcal{P}^+$, $\mathcal{Q}^-$ and the two are PC. Likewise for $\mathcal{P}^-$ and $\mathcal{Q}^+$.

\end{proof}

\begin{customlem}{\ref{lemma:hyperbox-pc}}
Let $\mathscr{P} = \{\mathcal{P}_1, \dots \mathcal{P}_k\}$ be a polyhedral complex and let $\mathcal{D}$ be any polytope. Then the set $\{\mathcal{P}_i \cap \mathcal{D} \; | \; \mathcal{P}_i \in \mathscr{P}\}$ also forms a polyhedral complex.
\end{customlem}

\begin{proof}
Letting $H_j$ be the hyperplanes that compose $\mathcal{D}$, i.e., $\mathcal{D} = \bigcap_j H_j$. Then it suffices to show that $\{\mathcal{P}_i \cap H_j \; | \; \mathcal{P}_i \in \mathscr{P}\}$ is a polyhedral complex, as we can repeat this iteratively for each $H_j$. This is equivalent to stating that for each $\mathcal{P}_i, \mathcal{P}_j \in \mathscr{P}$ with nonempty intersection, $\mathcal{P}_i\cap H_j$ and $\mathcal{P}_j\cap H_j$ are PC. This follows from a direct application of Lemma \ref{lemma:double-hyperplane-glued}.
\end{proof}

\begin{lemma}\label{lemma:intersection-check}
Let $\mathcal{P}$, $\mathcal{Q}$ be polytopes whose intersection is $(n-d)$ dimensional, for some $d\geq 2$, and let $x_0 \in \mathcal{P}$, with $B_t(x_0)$ the largest $\ell_p$-norm ball centered at $x_0$ contained in $\mathcal{P}\cup \mathcal{Q}$. Then $B_t(x_0)$ is contained entirely in $\mathcal{P}$.
\end{lemma}

\begin{proof}
First we state an equivalent representation of $B_t(x_0)$, 
\begin{equation}\label{eq:ellp-reformulation}
  B_t(x_0) = \bigcup\limits_{\{z \; | \; ||x_0-z|| \leq t\}} B_d(z) \; \; \;\text {for } d = (t- ||x_0-z||) 
\end{equation}
Certainly the $\subseteq$ inclusion holds by setting $z=x_0$ and the $\supseteq$ inclusion holds by the triangle inequality. Now let's assume that $\mathcal{P}\cap \mathcal{Q}$ is nonempty and contained in an $(n-2)$-dimensional linear subspace, $H$. Suppose for the sake of contradiction that $r > 0$ for 
$$r := \sup\limits_{x\in \mathcal{P}\cap \mathcal{Q}} t - ||x - x_0||$$ 
and $z$ is defined as some point in $\mathcal{P}\cap \mathcal{Q}$ that attains this supremal distance. Such a $z$ must exist because $\mathcal{P}\cap \mathcal{Q}$ is closed. Then $B_r(z) \subseteq B_t(x_0) \subseteq (\mathcal{P}\cup\mathcal{Q}) \subseteq H$. But $B_r(z)$ is contains some $\ell_2$ ball, regardless of our choice of norm, contradicting the previous chain of inclusions. Thus $r\leq 0$, indicating that $B_t(x_0) \subseteq \mathcal{P}$.
\end{proof}

%% file: supplementary_files/05_plnn_proofs.tex
\section{Geometry of Piecewise Linear Neural Networks}\label{app:polyhedral}
In this appendix we restate and prove our theorems regarding the geometry of PLNN's. Specifically, we prove our lemma which describes that each ReLU configuration defines a polytope and, in general position, its facets correspond to exactly one ReLU being flipped. Then we prove that the decision region forms a polyhedral complex. 
\subsection{Computing the linear region of neural networks}
First we prove this lemma:
\begin{customlem}{\ref{lemma:neural}} For a given neuron configuration $A$, the following are true about $\mathcal{P}_A$,
\begin{enumerate} [label=(\roman*)]
    \item  Unless $\mathcal{P}_A = \mathbb{R}^n$ or $\emptyset$, there exists a neuron configuration $B$ such that $\mathcal{P}_A\cap \mathcal{P}_B \neq \emptyset$.
    \item $\mathcal{P}_A$ is a polytope, and for all layers $i$, $f^{(i)}(x)$ is linear in $x$ for all $x\in \mathcal{P}_A$.
\end{enumerate}
\end{customlem}
\begin{proof}
\textbf{Item \ref{neural:i}}: This is trivial as certainly every point in the domain corresponds to at least one neuron configuration. If both $\mathcal{P}_A$ and $\mathcal{P}_A^c$ are not the empty set, then their intersection is nonempty. But $\mathcal{P}_A^c$ is composed of a union of at least one piecewise linear region, at least one of which must intersect $\mathcal{P}_A$.

\textbf{Item \ref{neural:ii}}: This is easy to see by simply writing down the polytope $\mathcal{P}_A$ and its corresponding linear function. For neuron configuration $A$, we partition $A$ into $A_1, A_2, \dots A_{l-1}$, with $A_i$ corresponding to the neuron configuration at the $i^{th}$ layer. Then letting $\Lambda_i$ be a fixed matrix to replace each ReLU in the network, defined as $\Lambda_i:=\text{diag}(A_i)$ we note that 
$$
f^{(i)}(x) =
\begin{cases}
W_i x + b_i, & \text {if i = 1}\\
W_i \sigma(\Lambda_i) (f^{(i-1)}(x)) +b_i, & \text{if } i > 1  \\
\end{cases}
$$
Hence, as $\sigma(\Lambda_i)$ is constant across all points with neuron configuration $\mathcal{A}$, $f$ is a composition of linear functions and must be linear everywhere with that neuron configuration. To define the polytope $\mathcal{P}_A$, we note that each neuron adds one linear constraint to the polytope. Let us write down each of these constraints exactly. Since each $f^{(i)}(x)$ is linear, it can be written as $V_ix + c_i$ for some $V_i,c_i$. Recalling that $f^{(i)}(x)$ is the input to the $i^{th}$ ReLU layer, the constraints are of the form $f^{(i)}(x) \;\underbar{?}\; 0$ where $\underbar{?}$ is the comparator $\geq, \leq, =$ for $A_{i,j}$ being $1, -1, 0$ respectively. This can be encoded efficiently by multiplying the lefthand side by $-\Lambda_i$, so the total constraint becomes $\Lambda_i(V_ix + c_i) \geq 0$. We remark that $\Lambda_i$ can be computed with a single forward pass of the network, and each $V_i$ and $c_i$ can be computed with a two matrix multiplications, one of which is a diagonal matrix.

\end{proof}

\subsection{PLNN's Form Polyhedral Complices} 
We can now prove our main theorem regarding the linear regions of a PLNN.
\begin{customthm}{\ref{thm:relu_nets_are_PG}}
The collection of $\mathcal{P}_A$ for all $A$, such that $\mathcal{P}_A$ has nonempty interior forms a polyhedral complex. Further, the decision region of $F$ at $x_0$ also forms a polyhedral complex. 
\end{customthm}
\begin{proof}
Let $\mathscr{P}_{i,j}$ denote the set of polytopes generated by neuron configurations of all neurons in layer $k <i$, and the first $j$ neurons in layer $i$. Let $\mathscr{P}_{i,0}$ refer to the set of polytopes generated by neuron configurations from all neurons in layer $k < i$. We'll prove the theorem by induction across $i$, with an inner induction on $j$. 

As a base case, consider only the first layer $f^{(1)}(x)$. Examining only neuron $j$ of the first layer, note that $f^{(1)}(x)_j = W_{1,j}x + b_{1,j}$ implying that the, unless $W_{1,j} =0$, the set of inputs $x$ for which $f^{(1)}(x)_j = 0$ is exactly a hyperplane, which we shall denote $H_j$. Then we can perform a second, interior, induction across the neurons of the first layer of $f$. 

The first neuron in the first layer separates $\mathbb{R}^n$ into two closed halfspaces, such that $\mathscr{P}_{1,1}$ is PC. Now assume that $\mathscr{P}_{1,k}$ is PC. Consider now the addition of the $(k+1)^{th}$ neuron to generate $\mathscr{P}_{1, k+1}$. In particular, if $\mathscr{P}_{1, k}$ is generated by considering the arrangement of hyperplanes $H_1, \dots H_k$, then $\mathscr{P}_{1, k+1}$ is $\mathscr{P}_{1, k}$ with the addition of hyperplane $H_{k+1}$. Letting $\mathcal{P} \mathcal{Q}$ be two PC polytopes in $\mathscr{P}_{1, k}$, we can let $H_{k+1}$ define $H_{\mathcal{P}}$ and $H_\mathcal{Q}$ and apply lemma \ref{lemma:double-hyperplane-glued} to demonstrate that the polytopes generated by this intersection remain PC. This concludes the base case of the outer induction.

Now let's assume that for any layer $k$, $\mathscr{P}_{k, 0}$ is a polyhedral complex. Consider the difference between $\mathscr{P}_{k,0}$ and $\mathscr{P}_{k,1}$. Let $G_1$ refer to the set of points $x$ for which $f^{(k)}_1(x)=0$, i.e. the first neuron of layer $k$ has pre-ReLU value exactly zero. Now by \ref{lemma:neural} part \ref{neural:iii}, $f^{(k)}(x)_1$ is linear in each $\mathcal{P}_A \in \mathscr{P}_{k, 0}$. Thus for each such $\mathcal{P}_A$, $G_1\cap \mathcal{P}_A$ is either the emptyset or a hyperplane, $H_A$. Any two polytopes $\mathcal{P}_A, \mathcal{P}_B$ contained in $\mathscr{P}_{k, 0}$ with nonempty intersection, by inductive assumption, must be PC. If $H_A\cap F \neq \emptyset$, then certainly $G_1 \cap \mathcal{P}_B \neq \emptyset$ and thus there must be some hyperplane $H_B$ such that $H_B=\mathcal{P}_B\cap G_1$. Since $F \cap G_1 = H_A \cap F$ and $F\cap G_1 = H_B \cap F$, we meet the criteria to apply lemma \ref{lemma:double-hyperplane-glued} and thus the polytopes generated by the addition of $G_1$ remain PC. 

To conclude the proof of the first statement in the theorem, assume that $\mathscr{P}_{k, j}$ is PC. Then consider the addition of the $(j+1)^{th}$ neuron of layer $k$. Let $G_{j+1}$ refer to the set of points for which $f^{(k)}_{j+1}(x)=0$. Note that $f^{(k)}_{j+1}$ is linear across each $\mathcal{P}_A \in \mathscr{P}_{k, 0}$, since we just as well could have initially incorporated the $(j+1)^{th}$ neuron of this layer instead of the first one. Consider any pair of polytopes $\mathcal{P}_A, \mathcal{P}_B \in \mathscr{P}_{k, j}$ with nonempty intersection. These must be PC, and in particular their union must either be fully contained in some $\mathcal{P}_C\in \mathscr{P}_{k, 0}$ or not. If so, then there exists some hyperplane $H_C$ such that $G_{j+1} \cap \mathcal{P}_C=H_C\cap \mathcal{P}_C$ and thus $\mathcal{P}_A \cap \mathcal{P}_B \cap G_i = \mathcal{P}_A \cap \mathcal{P}_B \cap H_C$ so we satisfy the criteria to apply lemma \ref{lemma:double-hyperplane-glued}. If there is no such $\mathcal{P}_C$, then $\mathcal{P}_A \cap \mathcal{P}_B$ must be a facet of each of them, $F$. Then we can mimic the argument in the previous paragraph to show that the polytopes generated by the addition of $G_{j+1}$ remain PC. 

Finally, we need to prove that the decision region of $F$ at $x_0$ forms a polyhedral complex. Let $\mathscr{Q}$ be the collection of linear regions of $F$ that have a nonempty intersection with the decision region of $F$ at $x_0$. As any subset of a polyhedral complex is also a polyhedral complex, $\mathscr{Q}$ is certainly a polyhedral complex. Let $F(x_0) = i$ and let $g_j = \{x| f_i(x) \geq f_j(x)\}$. For each linear region of $f$, $g_j$ is a halfspace. The decision region of $F$ at $x_0$ is exactly $\{\mathcal{Q}_i \cap (\bigcap_{j \neq i} g_j \; | \; \mathcal{Q}_i \in \mathscr{Q}\}$. It suffices to show that for a single $j$, $\{\mathcal{Q}_i \cap g_j(x)) \; | \; \mathcal{Q}_i \in \mathscr{Q}\}$ is still a polyhedral complex, as we can iterate over all $j\neq i$. Then for a fixed $j$ and any $\mathcal{Q}_i, \mathcal{Q}_k \in \mathscr{Q}$ with nonempty intersection, and letting $g_j(\mathcal{P})$ be the hyperplane defining $g_j(x)$ for the linear region $\mathcal{P}$, we note that $\mathcal{P}\cap\mathcal{Q}\cap g_j(\mathcal{P}) =\mathcal{P}\cap\mathcal{Q}\cap g_j(\mathcal{Q})$. This is exactly the criteria required to apply lemma \ref{lemma:double-hyperplane-glued}, which maintains that the pair of polytopes $\mathcal{P}$ and $\mathcal{Q}$ lying in the decision region are PC. This holds for every pair of polytopes in $\mathcal{Q}$ with nonempty intersection, so $\mathscr{Q}\cap g_j$ is a polyhedral complex, and hence so is the entire decision region of $F$ at $x_0$.
\end{proof}

In fact, the following corollary demonstrates that except in extreme cases, the facets of each linear region correspond to exactly one neuron flipping configurations.
\begin{customcor}{\ref{cor:relu-hamming}}
If the network parameters are in general position and $A,B$ are neuron configurations such that $dim(\mathcal{P}_A)=dim(\mathcal{P}_B)=n$ and their intersection is of dimension $(n-1)$, then $A,B$ have hamming distance 1 and their intersection corresponds to exactly one ReLU flipping signs.
\end{customcor}
\begin{proof}
As both $\mathcal{P}_A$ and $\mathcal{P}_B$ are of full dimension, no coordinate of the neuron configurations $A,B$ can be zero.  Under the assumption of general position of the network parameters, the halfspace that defines each polytope constraint lies in a different $(n-1)$-dimensional affine subspace, hence each facet corresponds to exactly one neuron. Indeed, each facet of each linear region's polytope corresponds to at exactly one ReLU constraint being set to equality. Since $dim(\mathcal{P}_A\cap\mathcal{P}_B)=n-1$ and since $\mathcal{P}_A, \mathcal{P}_B$ are PC, $\mathcal{P}_A, \mathcal{P}_B$ must be a facet of each of them. This facet is a linear region of the network as well, corresponding to a neuron configuration $C$ that is identical to $A, B$, but with some coordinate set to zero. As $A \neq B$, and the neuron configuration $C$ has exactly one zero, it must be the case that the hamming distance between $A$ and $B$ is exactly one, corresponding to exactly one ReLU flipping signs.
\end{proof}

%% file: supplementary_files/06_upper_bounds.tex
\section{An Approach For Computing Tighter Upper Bounds}\label{app:upper_bounds}
As mentioned in Section \ref{sec:speedups}, maintaining a nontrivial upper bound on the pointwise robustness accelerates the runtime of GeoCert by restricting the domain we have to search. This has a twofold benefit as (i) this allows us to quickly reject potential facets as infeasible by checking if their containing hyperplane intersects the restricted domain, and (ii) allows for tighter pre-ReLU activation bounds to be computed. This latter point allows for potential facets to be rejected without the computation of their projection as Corollary \ref{cor:relu-hamming} implies that neurons that are stable within a domain do not correspond to any facets inside that domain.

Fortunately, there has been an explosion in the field of computing upper bounds to the pointwise robustness, typically described as adversarial examples. In this section we present a variant of the attack techniques presented in \cite{kurakin2016adversarial, tramer2017ensemble, madry2017towards, cwattack}. Our goal is to be able to compute a reasonably tight upper bound for a single example in a very short amount of time. In general, attack techniques are viewed as optimizations over some perturbation that aims to maximize a loss that is large when the classifier makes a mistake. We discuss two popular existing adversarial attacks from an .

One attack, known as PGD performs \emph{gradient ascent} directly on the loss ands projects at each iteration back onto a set of allowable perturbations. Letting the allowable set of perturbations be $B_p^\epsilon(0)$ and the domain of valid images be $\mathcal{D}$, then the allowable set of adversarial perturbations for image $x_0$ is $\mathcal{D}':= B_p^\epsilon(0) \cap \{x - x_0 \; | \; x \in \mathcal{D}\}$. PGD seeks to solve the maximization problem
\begin{equation}
    \max\limits_{\delta \in \mathcal{D}'} \mathcal{L}(x_0+\delta, y)
\end{equation} 
where $\mathcal{L}(\cdot, y)$ is some loss that is small when the network classifies its argument as class $y$, and large otherwise. The PGD iterations become 
\begin{equation}
    \delta^+ = \Pi_{\mathcal{D}'}\Big(\delta + \eta \nabla_\delta \mathcal{L}(x_0 + \delta, y)\Big)
\end{equation}
Notice that the goal of PGD is not to induce a minimal distortion adversarial example, but simply to minimize classifier accuracy within a fixed threat model. We also note several tricks that are useful in practice such as a random initialization of $\delta \in \mathcal{D}'$ and repeated restarts to find more successful adversarial examples.

An alternative attack, pioneered by Carlini and Wagner \cite{cwattack} does aim to produce low-distortion adversarial examples by simply letting $\mathcal{D}':= \{x - x_0 \; | \; x \in \mathcal{D}\}$ and solving the optimization 
\begin{align} \label{eq:cw-attack}
\min_{\delta \in \mathcal{D}'} &\quad ||\delta|| & \\
    \text{s.t.} & F(x_0+\delta) \neq F(x_0) \notag \\ 
\end{align}
Where the adversarial constraint is typically put into the lagrangified form with the best multiplier found via binary search:

\begin{equation}
\min_{\delta \in \mathcal{D}'} \quad ||\delta|| + \lambda G(x_0+\delta)
\end{equation}
Where $G$ is a function that is zero everywhere where the classifier makes a mistake, and positive elsewhere. This is then solved with a standard gradient descent algorithm. The main critique of this method is that the binary search over the hyperparameter $\lambda$ dictates the runtime be several times longer than PGD. Note that during this optimization, once the intermediate iterate is outside $x_0$'s decision region, the gradient steps push the intermediate iterate radially inwards. However, unless step sizes are tuned nicely, many iterations with the radially-inward direction may be taken.

We provide a tweak to PGD that allows one to quickly generate adversarial examples that are optimized to have minimal distortion. This technique is as follows: for example image $x_0$, compute many random perturbations on $x_0$, and run PGD with a large domain on each of these randomly perturbed starting points. Once complete, collect each of the examples for which the classifier makes a mistake. Run a binary search along the line connecting the example and the starting point $x_0$, in an attempt to `project' onto the decision boundary. Return the minimal-distance of these projected adversarial attacks as the adversarial example for $x_0$. 

The binary search step requires only forward passes and is significantly faster than the several gradient steps required by CW to `project' back to the decision boundary. This allows one to effectively perform a quick PGD attack, which is almost always successful under a sufficiently large threat model, but also attain a successful adversarial attack with small distortion.

We note, the emphasis here is not on attaining the minimal distortion adversarial example, but on speed and guaranteed success. Our goal is to very quickly find an adversarial example that is incentivized to be close to the original point and will almost always succeed.

%This algorithm is outlined in \textcolor{red}{COPY ALGORITHM}

\begin{algorithm}[tb]
   \caption{Fast Upper Bound}
   \label{alg:fastupper}
 \hspace*{0.0em} \textbf{Input} classifier $f$, input $x_0$, initSize $\nu$, ballSize $\epsilon$\\
 \hspace*{2.8em} lr $\eta$, numIter $n$, numRand $r$\\
 \hspace*{2.8em} numBin $k$
\begin{algorithmic}
   \FOR{$i \in [r]$}
     \STATE $u_i = \infty$ \\ 
     \STATE $\delta_i \gets RandBall(\nu)$ \\ 
      \FOR{$iter \in [numIter]$}
        \STATE $\delta_i \gets \Pi_{\epsilon}(\delta_i + \eta \nabla f(x+\delta_i))$\\ 
      \ENDFOR
      \IF{$f(x+\delta_i) \neq f(x)$}
        \STATE $\delta_i \gets BinSearch(f, x_0, \delta_i, k)$\\
        \STATE $u_i \gets ||\delta_i||_p$\\
      \ENDIF\\
   \ENDFOR 
\end{algorithmic}
 \hspace*{0.0em} \textbf{RETURN} $\min_i u_i$
\end{algorithm}

\begin{algorithm}[tb]\label{alg:fastUpper}
   \caption{BinSearch}
   \label{alg:binsearch}
 \hspace*{0.0em} \textbf{Input} classifier $f$, point $x_0$\\
 \hspace*{2.8em} perturbation $\delta$, numIter $n$
\begin{algorithmic}
    \STATE $lo \gets 0,\; hi \gets 1$\\
    \FOR{$i \in [n]$}
        \IF{$f(x_0 + (lo + hi)/2 \cdot \delta) \neq f(x_0)$}
            \STATE $hi \gets (lo + hi) / 2$\\ 
        \ELSE
            \STATE $lo \gets (lo + hi) / 2$\\ 
        \ENDIF
    \ENDFOR 
\end{algorithmic}
 \hspace*{0.0em} \textbf{RETURN} $hi \cdot \delta $
\end{algorithm}

%% file: supplementary_files/07_extra_experiments.tex
\section{Extra Experiments}\label{app:experiments}

\subsection{Extra Experiment 1:}

To reiterate, in the worst case our algorithm may need to explore an exponential number of polytopes. Here, we provide results which seem to suggest that for PLNNs trained on MNIST the number of polytopes is well removed from the worst case. Figure \ref{fig:extra_1} shows the number of polytopes encountered in an $\ell_\infty$ ball of size $t$ around several random images. (Note that the relevant network in this case is the 70NetBin network described previously.) The distance $t$ is increased until the region around each of the sampled points includes the entire domain for MNIST (i.e. [0, 1] hypercube). Thus, the maximum number of polytopes that could be encountered for this problem is very loosely upper bounded by 73. On average, the number of polytopes encountered for this example would be closer to 6 as the average distance is 0.19. This plot seems to suggest that the number of polytopes encountered is much smaller than the worst case possibility.

\begin{figure}
    \centering
    \includegraphics[scale=0.5]{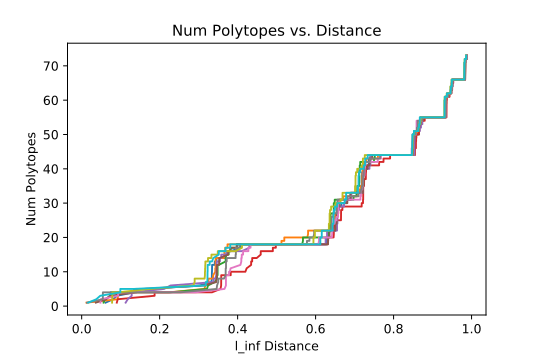}
    \caption{Evidence that verification for trained nets does not follow worst case behavior}
    \label{fig:extra_1}
\end{figure}

\subsection{Extra Experiment 2:}

Additionally, we run experiments to investigate the benefit of using a Lipschitz overapproximation based potential versus the standard $\ell_p$ distance. Table \ref{table:pot-table} demonstrates the average number of encountered polytopes when verifying pointwise robustness. 

\begin{table}[h!]
\centering
\label{table:pot-table}
\caption{Average number of polytopes explored until computing exact pointwise robustness across binary (1's and 7's only) MNIST, and full MNIST, and two architectures. The average is over 50 random examples. This demonstrates the benefit of leveraging the Lipschitz upper bound in the potential function.}
\begin{tabular}{@{}l|rr|rr@{}}
\toprule
                 & \multicolumn{2}{c}{Binary MNIST} & \multicolumn{2}{c}{Full MNIST} \\ \midrule
Potential        & 70Net           & 40Net          & 70Net          & 40Net         \\ \midrule
$\phi_{lip}$     & 4.2             & 15.3           & 9.7            & 27.5          \\
$\phi_{p}$ & 5.1             & 25.6           & 17.1           & 90.3          \\ \bottomrule
\end{tabular}
\end{table}